\documentclass[11pt,letterpaper]{article}
\usepackage[margin=1in]{geometry}
\usepackage[T1]{fontenc}
\usepackage{lmodern}
\usepackage{amsmath,amssymb,mathtools,amsthm}
\usepackage{microtype}
\usepackage{graphicx,booktabs}
\usepackage{algorithm,algpseudocode}
\usepackage[round]{natbib}
\usepackage{xcolor}
\usepackage[hidelinks]{hyperref}
\newtheorem{lemma}{Lemma}
\newtheorem{theorem}[lemma]{Theorem}
\newtheorem{corollary}[lemma]{Corollary}
\newtheorem{proposition}{Proposition}[section]
\newtheorem{claim}[proposition]{Claim}
\newcommand{\E}{\mathbb E}
\newcommand{\R}{\mathbb R}
\newcommand{\Law}{\operatorname{Law}}
\newcommand{\norm}[1]{\lVert#1\rVert}
\newcommand{\arxivdir}{}
\IfFileExists{arXiv version/figures/gaussian-contours.pdf}{\renewcommand{\arxivdir}{arXiv version/}}{}
\hypersetup{
  pdftitle={Two-way Guidance: Simultaneous Learning of Summary Statistics for Generative Bayes},
  pdfauthor={Percy S. Zhai, Max Pachman, Peverel D. Shipley, Jungeum Kim}
}
\title{Two-way Guidance: Simultaneous Learning of\\ Summary Statistics for Generative Bayes}
\author{Percy S.~Zhai\textsuperscript{1}\quad Max Pachman\textsuperscript{2}\quad
Peverel D.~Shipley\textsuperscript{2}\quad Jungeum Kim\textsuperscript{2}\\[0.7em]
\small \textsuperscript{1}University of Chicago\\
\small \textsuperscript{2}North Carolina State University}
\date{}
\begin{document}
\maketitle
\begin{abstract}
Bayesian simulation-based inference (SBI) can struggle to recover increasingly concentrated posteriors as observation size grows.
We consider learning sufficient summary statistics \emph{jointly} with the posterior as a general principle for scalable inference in various SBI methods.
A common measure taken in SBI is dimension reduction through predetermined summaries, which may partially discard posterior information.
Our approach instead imposes structure on how the posterior learner depends on a learned summary (e.g., affinity), and learns the summary simultaneously with the posterior.
We identify an important mechanism named two-way guidance: under a structural restriction on the summary, the posterior sufficiency of the learned summary simplifies posterior learning, whereas a suitable posterior objective guides the learned summary towards posterior sufficiency.
This mechanism explains a counterintuitive finding that the structural restriction on the summary statistics improves posterior learning, rather than limit its performance.
We establish theoretical guarantees on both posterior estimation error and the learned summary's departure from sufficiency, thus identifying conditions on the simulation budget for recovering increasingly concentrated posteriors as observation size grows.

\end{abstract}
\section{Introduction}
\label{sec:introduction}

Bayesian simulation-based inference (SBI) aims to recover the posterior distribution $\pi(\theta\mid X)$ when data can be simulated from a parameter $\theta$, but the likelihood is unavailable for evaluation.
A challenging setting arises when the observation $X$ is high-dimensional and only a handful of parameters govern its distribution.
Such problems arise in dynamical systems and population models \citep{cranmer2020frontier, gloeckler2025compositional}.
For instance, observing a time series for longer should improve inference about its governing parameters.
In models where information accumulates, the posterior becomes increasingly concentrated as the observation size $n$ grows.
The goal is a posterior learner that recovers this improved precision and correctly describes the remaining uncertainty.

We consider learning summaries that retain posterior information simultaneously with the posterior.
Our approach imposes structure on how the posterior learner depends on a summary $S=f(X)$, while allowing the summary itself to be learned.
Affinity assumption is one example: the posterior-learning function is linear in the learned summary coordinates; see \cite{kim2025quantile}.
The summary statistics organize posterior learning, and the posterior objective guides the learned summary not to lose the information needed for inference.
This two-way guidance is the mainline of this paper, and we study when this principle supports posterior accuracy even with a growing $n$.

Neural likelihood estimation, quantile learning, adversarial inference, and flow matching provide different ways to perform posterior inference from simulated data \citep{papamakarios2019snl, kim2025quantile, wang2026generative, zhai2026cfm}.
In these methods, an increasing observation dimension remains a challenge.
Compositional methods address it by exploiting independence or Markov structure to combine local inference problems \citep{geffner2023compositional, gloeckler2025compositional}.
On the theoretical side, \citet{frazier2024accuracy} characterize simulation requirements for posterior accuracy with a specified summary statistic as the observation size grows.
When the summary is learned, its information loss and learning cost must also enter this analysis.

The value of learning informative summaries has been discussed in the literature.
In approximate Bayesian computation (ABC), prefixed summaries facilitate data comparisons, but require model knowledge and can discard posterior information \citep{hoffmann2026unifying}.
Automatic approaches instead learn summaries for an inferential objective.
Semi-automatic ABC targets posterior means \citep{fearnhead2012semi}.
BayesFlow jointly trains summary and posterior networks \citep{radev2020bayesflow}.
Information-based methods learn representations that can be used in ABC and neural likelihood inference \citep{chen2021neural}.
BayesFlow also establishes an ideal-convergence result connecting joint learning to correct posterior sampling, while \citet{hoffmann2026unifying} organize summary selection through expected posterior entropy.

A particularly useful example is quantile learning. Building on affine dependence in learned coordinates \citep{rosenberg2023nonlinear}, \citet{kim2025quantile} jointly learn a summary and posterior quantiles under an affinity assumption, i.e., the true quantile map being linear to the summary statistic.
A meaningful follow-up question is why an apparent restriction on the class of summary statistics can improve posterior learning compared with using the raw data.
This question motivates a broader study of structured summary learning across SBI methods, an understudied topic in the literature.
Our central contribution is to establish a connection between structured summary learning and SBI scalability.
We explain when imposing a simple dependence on a learned summary can make posterior learning more efficient, and how the imposed structure can be automatically learned.

Throughout this paper, the notion \emph{posterior sufficiency} of summary $f$ means preservation of the posterior,
\begin{equation}
\pi(\cdot\mid X)=\pi(\cdot\mid f(X))
\quad\text{almost surely}.
\label{eq:posterior-sufficiency}
\end{equation}
Note that classical definition of sufficiency is characterized by Neyman-Fisher factorization theorem.
The definition in \eqref{eq:posterior-sufficiency} is under one specified prior, whereas classical sufficiency implies posterior sufficiency for every prior.
In this paper, an important property of such sufficient summary is what we call \emph{two-way guidance}.
On the one direction, sufficient summary simplifies posterior learning by preserving the information needed for inference and organizing it into low-dimensional coordinates in which the posterior-learning function has a simpler structure.
A less intuitive finding is that the posterior objective, in turn, guides the summary towards sufficiency.
To see this, note that simultaneous learning of summary and posterior incurs two extra components in posterior error: loss of information from replacing raw $X$ by summary $f(X)$, and a posterior learning gap based on $f(X)$.
We shall verify that the first gap vanishes exactly when \eqref{eq:posterior-sufficiency} holds under mild conditions.
Therefore, an insufficient summary incurs a penalty that cannot be removed by improving the posterior network alone.
Joint learning can reduce that penalty by guiding the summary toward losing less information.
We shall see that this two-way guidance intuition applies to multiple SBI methods.

At first sight, imposing a structural constraint (like aforementioned affinity assumption) appears to cause an additional limitation on summary learning.
Note, however, that a sufficient summary is not unique.
Sufficiency concerns the posterior information retained by the summary's $\sigma$-algebra.
Posterior learning also depends on how that information is represented.
An invertible change of summary coordinates preserves the information but can change the difficulty of how it guides posterior learning.
When the coordinates are learned jointly, nonlinear features of $X$ can be absorbed into $f$, allowing a simpler dependence on observation via summary.
We would like to point out that we do not need to recover a specific predetermined sufficient statistic, but rather any informative representation in which the posterior-learning function is simple with respect to the summary.
This is exactly what explains the counterintuitive benefit of a structural restriction.
The benefit of saving simulation budget can outweigh the cost of learning the summary and its structural error.
We shall quantify these requirements relative to posterior uncertainty as $n$ grows.

Our main contributions are as follows.
\begin{itemize}
\item We establish a theoretical connection between structured summary learning and SBI scalability.
We explain how the two-way guidance work, i.e., how the sufficiency of summary statistic and the accuracy of posterior learning reinforce each other.
We also explain how jointly learned coordinates counterintuitively make structural restrictions such as affinity beneficial for posterior learning.
\item
For Bayesian SBI objectives whose excess loss controls posterior error, we establish finite-simulation guarantees for both full-posterior accuracy and the learned summary's departure from sufficiency.
The bounds account for joint approximation, estimation, and optimization errors, quantifying how approximate sufficiency emerges under explicit conditions.
\item We identify sufficient simulation budgets for recovering increasingly concentrated posteriors. Comparing raw conditioning with our proposed method, we count the cost of summary learning and give sufficient conditions on the simulation budget $N(n)$ needed to correctly recover the shrinking posterior, which exhibits the benefit of summary learning.
\end{itemize}

\section{Structured Summary Learning}
\label{sec:methods}

Let $\theta\in\R^d$ denote the parameter and $X=(x_1,\ldots,x_n)$ the observation, where $x_i\in\R^{d_X}$.
Training pairs $(\theta,X)$ are generated from the prior and simulator.
We learn a summary $S=f(X)\in\R^q$ together with a posterior sampler.
Write $U\sim F_U$ for a source variable independent of $X$, and $G(u,s)$ for a map from source noise and a summary to the parameter space.
The desired map retains the full posterior through
\begin{equation}
\Law\{G(U,f(X))\mid X\}=\pi(\cdot\mid X).
\label{eq:source-posterior}
\end{equation}
Thus, posterior sampling consists of evaluating $G$ at independent source draws and the summary of the observed dataset.
The fitted map approximates this equality; its accuracy depends on both the information retained by $f$ and the posterior learned from it.

We describe how summary learning helps three SBI methods: Deep Generative Quantile Bayes (DGQB) \citep{kim2025quantile}, conditional flow matching \citep{zhai2026cfm}, and Bayesian GAN \citep{wang2026generative}.
The summary statistics enter different functions: the quantile potential, the joint velocity, and the conditional generator, respectively.
In each case, the summary is simultaneously learned through an objective involving the simulated parameter.
We write $f_\eta$ when the encoder parameters need to be distinguished from those of the posterior learner.
Appendix~\ref{app:original-methods} introduces the original methods and their relation to the objectives below.

\subsection{Deep Generative Quantile Bayes}
\label{sec:qr}

DGQB represents a posterior sampler as $G(u,f(X))=\nabla_u\psi(u,f(X))$, where the conditional potential $\psi$ is convex in its source argument $u$ \citep{kim2025quantile}.
The dependence on the summary can be specified through the architecture of $\psi$.
Once the networks are trained, differentiating the potential at fresh source draws gives posterior samples without a Markov chain or numerical integration.

For a minibatch of $m$ simulated pairs $(\theta_i,X_i)$ and independent source draws $U_j$, the conditional quantile objective takes the form
\begin{equation}
\widehat{\mathcal L}_Q(\psi,f)
=\frac1{m^2}\sum_{i=1}^m\sum_{j=1}^m\psi(U_j,f(X_i))
+\frac1m\sum_{i=1}^m\max_{1\le j\le m}
\big\{U_j^\top\theta_i-\psi(U_j,f(X_i))\big\}.
\label{eq:qnn-general-loss}
\end{equation}
This is a sample approximation of the conditional transport objective \citep{carlier2016vqr,sun2022conditional}, with its conjugate supremum replaced by a sample maximum.
The first term averages over source--summary pairs, reflecting independence of $U$ and $X$.
The summary $f$ and potential $\psi$ are learned simultaneously.

\begin{samepage}
Specifically, \citet{kim2025quantile} use the affinity assumption
\begin{equation}
\psi(u,f(X))=\varphi(u)+b(u)^\top f(X),
\label{eq:qr-potential}
\end{equation}
\end{samepage}
where $\varphi$ is scalar-valued and $b$ takes values in $\R^q$.
With centered batch summaries, $\sum_i f(X_i)=0$, \eqref{eq:qnn-general-loss} reduces to their Equation~(3.4),
\begin{equation}
\widehat{\mathcal L}_Q
=\frac1m\sum_{i=1}^m\bigg[\varphi(U_i)
+\max_{1\le j\le m}\big\{U_j^\top\theta_i-\varphi(U_j)
-b(U_j)^\top f(X_i)\big\}\bigg].
\label{eq:qr-loss}
\end{equation}
Centering leaves only the intercept outside the maximizing term, and the encoder, intercept, and slope functions are updated together.
Affinity is only one choice within \eqref{eq:qnn-general-loss}.
Other choices include nonnegative summary-dependent combinations of convex potentials and partially input-convex neural networks \citep{amos2017input}, which permit nonlinear summary dependence while preserving convexity in $u$.
We use the affine special case to explain the two-way guidance and present its training procedure in Algorithm~\ref{alg:qr}.

The two-way guidance intuition, taking affinity assumption as an example, can be seen directly in \eqref{eq:qr-loss}.
Once we obtain a sufficient summary admitting this affine dependence, the task of learning $\psi$ jointly in $(u,s)$ reduces to learning $\varphi$ and $b$ as functions of $u$ alone, without losing posterior information.
In the reverse direction, the maximizing term compares the potential with the simulated parameter $\theta_i$ and sends its gradient to $f(X_i)$ through $b(U_j)$.
If $f$ discards a feature of $X$ that changes the posterior, this missing dependence incurs a population loss that cannot be eliminated by improving $\varphi$ and $b$ while keeping $f$ fixed.
To remove this part of the loss, joint optimization must change $f$ so that the potential $\psi$ can distinguish those posterior differences.
Thus, the objective guides the summary to sufficiency.

\begin{algorithm}[t]
\caption{DGQB with a learned affine summary}
\label{alg:qr}
\begin{algorithmic}[1]
\raggedright
\State Initialize $f$, $\varphi$, and $b$.
\State Draw a minibatch of $(\theta_i,X_i)$ and $U_i\sim F_U$.
\State Evaluate and center the summaries $f(X_i)$.
\State Update $(f,\varphi,b)$ jointly using \eqref{eq:qr-loss}.
\State Repeat Steps 2--4 until training ends.
\State Freeze the fitted normalization and compute $s^*=\widehat f(X^*)$ for the observed dataset $X^*$.
\State Draw fresh $U$ and return $\nabla_u\widehat\psi(U,s^*)$.
\end{algorithmic}
\end{algorithm}

The earlier specification of \citet{sun2022conditional} uses a potential that is affine in centered $X$.
When the true posterior quantile depends nonlinearly on $X$, this specification is violated.
Learning $f$, on the contrary, allows nonlinear dependence on $X$ to enter through the summary while retaining an affine dependence on the learned summary.

\subsection{Conditional Flow Matching}
\label{sec:sfm}

Flow matching learns a velocity field that transports a source distribution at $t=0$ to a target distribution at $t=1$ through an ordinary differential equation \citep{lipman2023flow}.
For posterior inference, \citet{zhai2026cfm} construct a joint map in observation and parameter space.
Replacing the observation block by $S=f_\eta(X)$, we aim to transport independent source variables $S_0\sim\mathcal N(0,I_q)$ and $U=\theta_0\sim\mathcal N(0,I_d)$ to the target joint law of $(f_\eta(X),\theta)$.
We learn the joint velocity
\begin{equation}
\dot S_t=v^S_\alpha(t,S_t),\qquad
\dot\theta_t=v^\theta_\beta(t,S_t,\theta_t).
\label{eq:sfm-flow}
\end{equation}
The block-triangular structure means that the summary trajectory evolves independently of the parameter state and guides the parameter trajectory.

To express the posterior map $G$ explicitly, let $F_t(s_0)$ denote the summary flow, defined by
\begin{equation}
\frac{\mathrm d}{\mathrm dt}F_t(s_0)
=v^S_\alpha\big(t,F_t(s_0)\big),\qquad F_0(s_0)=s_0.
\label{eq:sfm-summary-map}
\end{equation}
For a well-posed invertible flow, a target summary $s$ determines its source point $F_1^{-1}(s)$ by solving \eqref{eq:sfm-summary-map} backward.
Along the corresponding trajectory $S_t=F_t(F_1^{-1}(s))$, the parameter equation in \eqref{eq:sfm-flow} defines
\begin{equation}
G(u,s)=u+\int_0^1 v^\theta_\beta(t,S_t,\theta_t)\,\mathrm dt,
\qquad \theta_0=u.
\label{eq:sfm-conditional-map}
\end{equation}
Thus, $G(u,s)$ is the parameter endpoint of a joint flow in parameter and summary space.
If the joint endpoint transports the independent source to the target joint law, its triangular form implies $\Law\{G(U,s)\}=\pi(\cdot\mid S=s)$ for almost every target summary $s$ \citep{baptista2024conditional}.
When $f$ is sufficient, this is the full-data posterior in \eqref{eq:source-posterior}.

For training, set $S_1=f_\eta(X)$ for each simulated pair $(\theta_1,X)$, and independently draw a source pair $(S_0,\theta_0)$ and $t\sim\operatorname{Unif}[0,1]$.
Training uses the straight-line interpolants
\begin{equation}
S_t=(1-t)S_0+tS_1,\qquad
\theta_t=(1-t)\theta_0+t\theta_1,
\label{eq:sfm-interpolants}
\end{equation}
and the regression losses
\begin{equation}
\begin{aligned}
\mathcal L_\theta(\eta,\beta)
&=\E\norm{v^\theta_\beta(t,S_t,\theta_t)-(\theta_1-\theta_0)}^2,\\
\mathcal L_S(\alpha;\eta)
&=\frac1q\E\norm{v^S_\alpha(t,\operatorname{sg}(S_t))
-\operatorname{sg}(S_1-S_0)}^2.
\end{aligned}
\label{eq:sfm-losses}
\end{equation}
Here $\operatorname{sg}(z)$ is the stop-gradient operation: it returns $z$ but contributes zero gradient during backpropagation.
The parameter loss updates $(\eta,\beta)$ together, whereas the summary loss updates only $\alpha$.

To demonstrate the two-way guidance in conditional flow matching, we may assume, for example, that the velocity functions in \eqref{eq:sfm-flow} are affine to $S_t$.
Besides reducing dimension, this structure prescribes how the velocities depend on the summary, leaving the networks parametrized by $\alpha$ and $\beta$ to learn the corresponding intercept and slope functions unrelated to observation, which is easier.
Conversely, if $f$ discards any feature that improves the prediction of parameter displacement $\theta_1 - \theta_0$, adjusting the velocity coefficients alone cannot recover the missing information.
Backpropagating the prediction error through $S_t$ allows joint optimization to reduce this loss by changing $f$ to retain such features, thereby directing summary learning toward sufficiency.

Note that in \eqref{eq:sfm-losses} the summary-flow loss is detached from the encoder.
Since the summary-flow loss only asks how easily source noise can be transported to the summary distribution, allowing it to update the summary $f$ could encourage the latter to discard information simply to make the distribution easier to transport, potentially collapsing the summary.
Detaching this loss leaves the encoder updates to the parameter-displacement objective.

\begin{algorithm}[t]
\caption{Summary flow matching}
\label{alg:sfm}
\begin{algorithmic}[1]
\raggedright
\State Initialize $f_\eta$, $v^S_\alpha$, and $v^\theta_\beta$.
\State Draw simulated pairs, independent source noise and times; form \eqref{eq:sfm-interpolants}.
\State Cache $\operatorname{sg}(S_1)$ and update $(\eta,\beta)$ using $\widehat{\mathcal L}_\theta$.
\State Update $\alpha$ using $\widehat{\mathcal L}_S$, cached summaries, and fresh summary noise and times.
\State Repeat Steps 2--4 until training ends.
\State Solve \eqref{eq:sfm-summary-map} backward from $f_{\hat\eta}(X^*)$ to $s_0^*$ and draw fresh $U$.
\State Integrate \eqref{eq:sfm-flow} forward from $(s_0^*,U)$ and return the parameter endpoint.
\end{algorithmic}
\end{algorithm}

The raw-data method learns a joint velocity in dimension $nd_X+d$ \citep{zhai2026cfm}.
Even though its parameter velocity has only $d$ outputs, it remains a regression function of the full observation state and the parameter state.
Increasing $n$ therefore increases the conditional input dimension and the size of the observation flow that must be learned and integrated.
Errors in that observation trajectory can also affect the parameter trajectory used for posterior sampling.
SFM learns the joint velocity in the summary and parameter space, whose dimension is $q+d$.
The encoder must still process the full dataset, and its information loss and learning cost remain part of the posterior-accuracy problem.

\subsection{Bayesian GAN}
\label{sec:sbgan}

B-GAN learns a conditional generator by distinguishing simulated pairs $(X,\theta)$ from generated pairs $(X,g_\beta(U,X))$ through an adversarial critic \citep{wang2026generative}.
Once trained, the generator provides independent draws from its fitted conditional distribution by direct evaluation at fresh source noise.
In SB-GAN, we introduce a learned summary through
\begin{equation}
G(u,s)=g_\beta(u,s)=a_\beta(u)+B_\beta(u)s,
\qquad s=f_\eta(X),
\label{eq:sbgan-generator}
\end{equation}
where $a_\beta(u)\in\R^d$ and $B_\beta(u)\in\R^{d\times q}$.
The encoder and the coefficient networks are learned together.
A separate critic $D_\omega(X,\theta)$ receives the original observation and may have its own encoder, with no parameters shared with $f_\eta$.

The population objective is
\begin{equation}
\min_{\eta,\beta}\sup_{D\in\mathcal D}
\E\big[D(X,\theta)-D\{X,g_\beta(U,f_\eta(X))\}\big].
\label{eq:sbgan-objective}
\end{equation}
For its ideal interpretation, $\mathcal D$ is the class of all jointly measurable critics that are 1-Lipschitz in $\theta$ for each fixed $X$.
With finite first moments and the usual measurability conditions, the inner supremum equals the average conditional $W_1$ distance between the full posterior and the generated distribution.
In practice, neural critic updates alternate with joint encoder and generator updates, using a gradient penalty in the parameter argument as described in Appendix~\ref{app:original-bgan}.

The two-way guidance also applies to this method.
The affine generator lets the summary organize posterior sampling through $q$ learned coordinates.
The critic supplies the reverse guidance by detecting discrepancies between simulated and generated parameters at the same observation.
Its gradient passes through $B_\beta(U)$ to the encoder, changing the summary according to errors in posterior generation.
Note that retaining the full observation in the critic is essential for this interpretation.
If the critic saw only the same learned summary, matching $\pi(\cdot\mid f_\eta(X))$ could achieve zero discrepancy even when the summary had discarded posterior information and collapsed to a point mass.
With the ideal full-data critic, an insufficient summary cannot achieve zero discrepancy.
A finite neural critic approximates this criterion, so the strength of this guidance depends on its ability to detect the relevant posterior differences.

\begin{algorithm}[t]
\caption{SB-GAN with a learned affine summary}
\label{alg:sbgan}
\begin{algorithmic}[1]
\raggedright
\State Initialize $f_\eta$, $(a_\beta,B_\beta)$, and a separate critic $D_\omega(X,\theta)$.
\State Take several critic steps on simulated and generated pairs, keeping $(\eta,\beta)$ fixed.
\State Draw a fresh simulated minibatch and source draws; set $\tilde\theta_i=g_\beta(U_i,f_\eta(X_i))$.
\State Update $(\eta,\beta)$ jointly to decrease $-m^{-1}\sum_iD_\omega(X_i,\tilde\theta_i)$, keeping $\omega$ fixed.
\State Repeat Steps 2--4 until training ends.
\State Freeze the fitted normalization and compute $s^*=f_{\hat\eta}(X^*)$.
\State Draw fresh $U$ and return $a_{\hat\beta}(U)+B_{\hat\beta}(U)s^*$.
\end{algorithmic}
\end{algorithm}

For a raw-data generator $g_\beta(U,X)$, both generation and discrimination must resolve dependence on a growing observation vector.
This can require greater network capacity and more simulated pairs, even when the parameter dimension is small.
SB-GAN separates the generator's data processing from its source-dependent coefficients, allowing one learned statistic to be reused across posterior draws.
The original B-GAN also considers DeepSets architectures; its pooled generator features can depend on both $U$ and $X$, unlike the statistic $f_\eta(X)$ used here.
The critic must still process the full observation, and its learning cost must be counted when evaluating the benefit of the summary.

\section{Posterior Accuracy and Learned Sufficiency}
\label{sec:learning}

We study whether joint learning recovers both the posterior and the sufficiency of the summary.
Let $N$ be the number of simulated training pairs and $(\widehat f,\widehat G)$ the fitted summary and posterior map.
Here $\E_N$ averages over the $N$ simulated training pairs and all training randomness, including auxiliary source draws and optimization.
The expectation $\E_X$ is over an independent observation from the prior predictive distribution.
With the fitted networks held fixed, write $\widehat\pi(\cdot\mid X)=\Law\{\widehat G(U,\widehat f(X))\mid X\}$ and define
\begin{align}
R_p&=\E_N\E_XW_2^2\{\widehat\pi(\cdot\mid X),\pi(\cdot\mid X)\},
\label{eq:v3-posterior-risk}\\
R_f&=\E_N\E_XW_2^2\{\pi(\cdot\mid X),\pi(\cdot\mid\widehat f(X))\}.
\label{eq:v3-summary-risk}
\end{align}
Here $W_2$ is the 2-Wasserstein distance, and the summary posterior also conditions on the fitted function held fixed.
We assume measurability and finite integrated second moments throughout.
Thus, $R_f=0$ means posterior sufficiency almost surely; candidate summaries are not assumed sufficient.

\begin{samepage}
\begin{lemma}[Posterior error and summary insufficiency]
\label{lem:v3-risk-decomposition}
For any fitted sampler depending on $X$ only through $\widehat f(X)$,
\begin{equation}
\begin{aligned}
R_p\leq{}&2R_f+2\E_N\E_XW_2^2\big\{\widehat\pi(\cdot\mid X),
\pi(\cdot\mid\widehat f(X))\big\},\\
R_f\leq{}&2R_p.
\end{aligned}
\label{eq:v3-risk-decomposition}
\end{equation}
\end{lemma}
\end{samepage}

The first bound separates summary insufficiency from fitting the oracle posterior given the summary.
The reverse bound is the cornerstone of two-way guidance: accurate posterior learning requires the summary to retain posterior information.

Let $\mathcal F$ be the summary class and $\mathcal G_f$ the admissible posterior learners for $f$.
Write $\mathcal L(f,G)$ and $\mathcal L_N(f,G)$ for their population and empirical objectives, and $\mathcal L_0$ for the common full-data population optimum.
Define the approximation error
\begin{equation}
\Delta_{\mathcal F}
=\inf_{f\in\mathcal F}\left\{
\inf_{G\in\mathcal G_f}\mathcal L(f,G)-\mathcal L_0\right\}\geq0.
\label{eq:v3-approximation}
\end{equation}
For fixed $n$ and classes, $\Delta_{\mathcal F}$ is fixed; $\Delta_{\mathcal F}=0$ requires a summary and compatible posterior structure that approach the full-data optimum.

Write $\pi_{f,G}(\cdot\mid X)=\Law\{G(U,f(X))\mid X\}$ and assume the calibration condition
\begin{equation}
\E_XW_2^2\{\pi_{f,G}(\cdot\mid X),\pi(\cdot\mid X)\}
\leq\Gamma\{\mathcal L(f,G)-\mathcal L_0\}.
\label{eq:v3-calibration}
\end{equation}
Here $\Gamma$ is a nondecreasing concave function on $[0,\infty)$, continuous at zero with $\Gamma(0)=0$.
It converts population loss error into posterior error under the method-specific conditions in Appendix~\ref{app:v3-calibration}.
Let the empirical optimization gap be
\begin{equation}
\xi_N=\mathcal L_N(\widehat f,\widehat G)
-\inf_{f\in\mathcal F,\,G\in\mathcal G_f}\mathcal L_N(f,G)\geq0.
\label{eq:v3-optimization}
\end{equation}
For adversarial objectives, $\mathcal L_N$ uses the critic supremum.

\begin{theorem}[Two-way guidance]
\label{thm:v3-two-way}
Fix the classes before training, with $(\widehat f,\widehat G)$ in these classes, and suppose \eqref{eq:v3-calibration} holds uniformly.
Assume the displayed suprema, infima and expectations are measurable and finite.
Then
\begin{equation}
R_p\leq\Gamma\Big[\Delta_{\mathcal F}+\E_N\xi_N
+2\E_N\sup_{f\in\mathcal F,\,G\in\mathcal G_f}
|\mathcal L_N(f,G)-\mathcal L(f,G)|\Big].
\label{eq:v3-end-to-end}
\end{equation}
For fixed $n$ and classes, if $\Delta_{\mathcal F}=0$, $\E_N\xi_N\to0$, and
\begin{equation}
\E_N\sup_{f\in\mathcal F,\,G\in\mathcal G_f}
|\mathcal L_N(f,G)-\mathcal L(f,G)|\longrightarrow0,
\label{eq:v3-uniform-consistency}
\end{equation}
then $R_p\to0$ and $R_f\to0$ as $N\to\infty$.
\end{theorem}

Under calibration, a summary that discards posterior information cannot approach the full-data optimum by improving its posterior learner alone.
Thus, sufficiency follows from successful learning.
Appendix~\ref{app:v3-theory-proofs} gives the proofs; sufficient conditions for statistical and optimization consistency appear in Appendix~\ref{app:v3-consistency}.

We now specialize to affinity: the function being learned has the form $a(w)+B(w)f(X)$, where $w\in\R^p$ contains its arguments other than the observation or summary.
For DGQB, $w=u$ and $p=d$; for SB-GAN, $p$ is the source dimension; for an affine SFM parameter velocity, $w=(t,\theta_t)$ and $p=d+1$.
The raw-data comparison learns a function of $(w,X)$, with $nd_X+p$ arguments.
Let $P_n$ count the encoder and auxiliary-network parameters; the affine coefficient functions are counted separately through $p$.

Assume bounded Lipschitz coefficient functions on bounded domains, bounded losses Lipschitz in the fitted functions, and bounded parameter ranges and uniform parameter sensitivity for the finite-dimensional networks.
Let $A_n\geq e$ collect these regularity bounds and the output and summary dimensions, as specified in Appendix~\ref{app:v3-affine-proof}.
Write $\Gamma_n$ for the calibration function with its dependence on $n$ retained, after normalizing the loss contributions to a common interval of length one.
The normalization scale remains in $\Gamma_n$.

\begin{corollary}[The learning cost under affinity]
\label{cor:v3-affine-rate}
Under Theorem~\ref{thm:v3-two-way}, assume $\Delta_{\mathcal F}=0$, the complete loss-cover conditions of Appendix~\ref{app:v3-affine-proof}, and the native reference-criterion conditions of Appendix~\ref{app:v3-native-rates}.
Require $\E_N\xi_N$ and the total objective-approximation error in \eqref{eq:v3-auxiliary-criterion-error} each to be no larger than the sum in square brackets below.
For $N\geq2$,
\begin{equation}
R_p\leq\Gamma_n\Bigg(10\Bigg[
\sqrt{\frac{(P_n+1)\log(2A_nN)}{N}}
+\left\{\frac{A_n^p\log(2A_nN)}{N}\right\}^{1/(p+2)}
\Bigg]\Bigg),
\label{eq:v3-affine-rate}
\end{equation}
and $R_f \leq 2R_p$.
For the raw-data class, replace $p$ by $nd_X+p$ and use that class's $P_n$ and $A_n$.
\end{corollary}

The two terms represent learning the summary and auxiliary networks, and learning the coefficient functions of $w$.
Affinity replaces the growing functional dimension $nd_X+p$ by $p$, while retaining the summary-learning cost in $P_n$ and $A_n$.
These are sufficient bounds for the stated classes, not a lower bound for every raw-data architecture.
For the implemented moving-summary SFM sampler, posterior calibration remains an additional requirement; regression convergence alone does not establish it.

\section{Experiments}
\label{sec:new-experiments}
We examine whether jointly learning a summary and a posterior learner improves inference as the observation size $n$ increases.
We compare DGQB, SFM, and SB-GAN with their raw-data counterparts.
Within each pair, we keep the simulation budget and training schedule fixed.

\subsection{Gaussian Conjugate Model}
\label{sec:new-gaussian}

We first consider a Gaussian conjugate model, whose closed-form posterior provides an explicit target for posterior learning.
Following \citet{kim2025quantile}, let $X_i\mid\mu,\sigma^2\stackrel{\mathrm{iid}}{\sim}N(\mu,\sigma^2)$, with priors
\[
\frac{\nu_0\sigma_0^2}{\sigma^2}\sim\chi^2_{\nu_0},
\qquad
\mu\mid\sigma^2\sim N(\mu_0,\sigma^2/\kappa).
\]
We set $(\mu_0,\sigma_0,\kappa,\nu_0)=(0,1,2,25)$ and generate observations at the fixed truth $(\mu,\sigma^2)=(0,1)$.
We increase $n$ over $\{64,128,256,512,1024,2048,4096\}$ and use a DeepSets summary with $q=32$.
As $n$ increases, the exact posterior becomes more concentrated, so we can examine whether the learned posterior follows this concentration.

\begin{figure}[!ht]
    \centering
    \includegraphics[width=0.95\linewidth]{\arxivdir 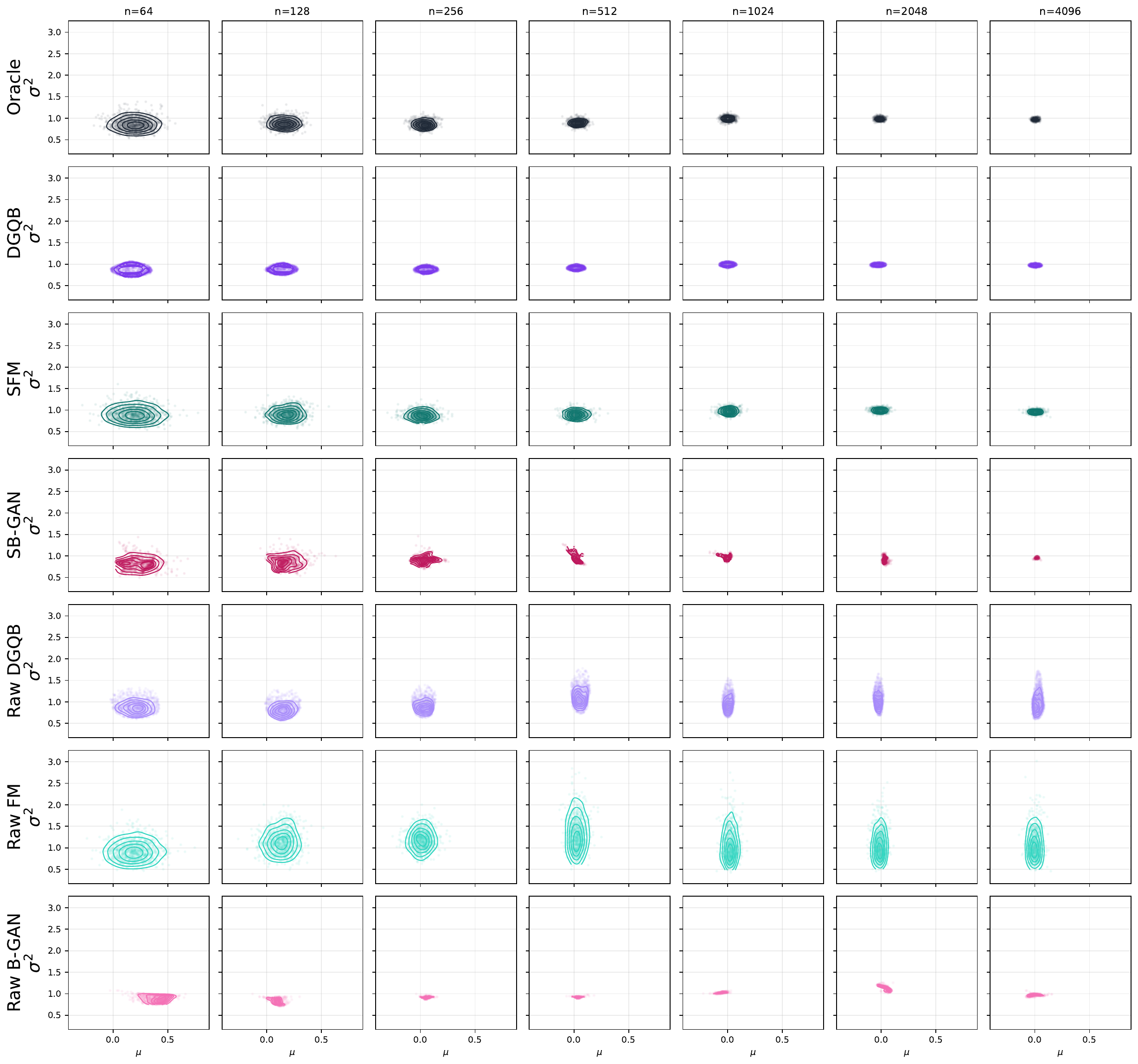}
    \caption{Gaussian posterior draws and contours, with the oracle posterior in the first row and increasing $n$ across columns.
    All panels use common $(\mu,\sigma^2)$ axes.}
    \label{fig:new-gaussian-contours}
\end{figure}

Figure~\ref{fig:new-gaussian-contours} displays the learned posterior contours as $n$ increases.
The summaries help the posterior learners follow the shrinking exact posterior, with a particularly visible improvement for DGQB and SFM.
To quantify this comparison, we compute the sliced Wasserstein-2 distance between samples from the learned and exact posteriors in the parameter space $(\mu,\sigma^2)$.
This distance averages squared one-dimensional Wasserstein-2 distances over random projection directions and takes the square root.
Note that the exact posterior conditions on all $n$ observations, so the comparison retains the full-data posterior as its target.
We use $1{,}000$ draws from each posterior and $256$ projection directions.

\begin{figure}[!t]
    \centering
    \includegraphics[width=0.96\textwidth]{\arxivdir 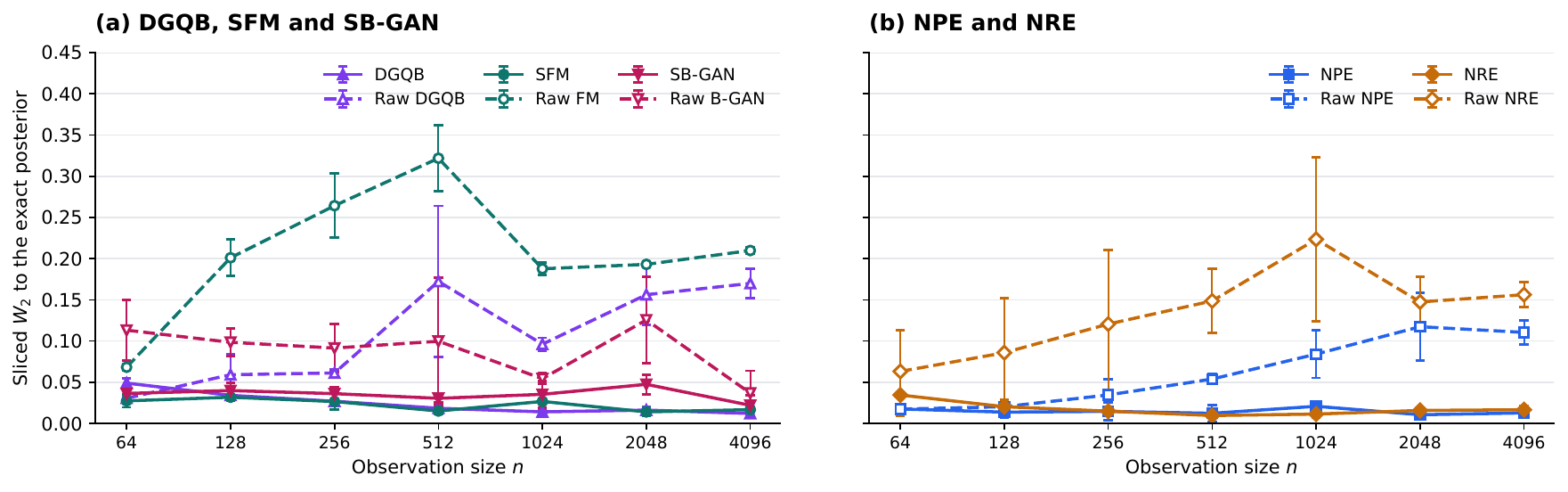}
    \caption{Gaussian posterior error: mean and SD across three seeds.
    Solid: summaries; dashed: raw data.}
    \label{fig:new-gaussian-swd}
\end{figure}

Figure~\ref{fig:new-gaussian-swd}(a) shows that joint summary learning improves posterior approximation at larger observation sizes, with larger gains for DGQB and SFM and a smaller gain for SB-GAN.
At $n=4096$, the mean distances are $0.012$, $0.017$, and $0.022$ for DGQB, SFM, and SB-GAN, compared with $0.170$, $0.210$, and $0.037$ for their raw-data counterparts.
The smaller improvement for SB-GAN is accompanied by remaining errors in posterior spread and location: its learned posterior is overly concentrated and its center is displaced from that of the exact posterior.
Thus, in this model, joint summary learning improves posterior approximation across different SBI methods and helps the learned posterior follow the concentration of the exact posterior as the observation size grows.

We demonstrate the wide applicability of joint summary learning mechanism by experiments on neural posterior estimation (NPE) and neural ratio estimation (NRE).
Their details are described in Appendices~\ref{app:original-npe} and~\ref{app:original-nre}.
At larger observation sizes, joint summary learning also yields smaller posterior approximation errors for both methods; see Figure~\ref{fig:new-gaussian-swd}(b).

\subsection{Brock--Hommes Model}
\label{sec:new-brock-hommes}

The Brock--Hommes model describes asset-price dynamics generated by traders with heterogeneous forecasting strategies \citep{brock1998heterogeneous}.
It provides a financial application in which the posterior is unavailable in closed form, and we train the posterior learners using simulations from the model.
Following the setup in \citet{kim2025quantile}, the price deviation $x_t$ from its fundamental value evolves as
\begin{equation}
\begin{aligned}
    x_{t+1}&=\frac{1}{R}\sum_{h=1}^{H}w_{h,t+1}(g_hx_t+b_h)
        +\epsilon_{t+1},\\
    w_{h,t+1}&=\frac{\exp(\beta U_{h,t})}{\sum_{j=1}^{H}\exp(\beta U_{j,t})},\\
    U_{h,t}&=(x_t-Rx_{t-1})(g_hx_{t-2}+b_h-Rx_{t-1}),
\end{aligned}
\label{eq:new-bh-model}
\end{equation}
where $g_h$ and $b_h$ are the trend and bias coefficients of strategy $h$, $U_{h,t}$ measures its past performance, and $w_{h,t+1}$ is its share of traders.
The constants $R$ and $\beta$ denote the gross risk-free return and the intensity of switching between strategies, respectively, and $\epsilon_t\stackrel{\mathrm{iid}}{\sim}N(0,\sigma^2)$.
We use $H=4$, $R=1.01$, $\beta=120$, and $\sigma=0.04$, with $g_1=b_1=b_4=0$ and $g_4=1.01$.
The unknown parameter is $\theta=(g_2,b_2,g_3,b_3)$, with independent uniform priors on $[0,1]$ for the first three coordinates and $[-1,0]$ for $b_3$.

We fix the truth at $\theta_0=(0.9,0.2,0.9,-0.2)$ and consider series lengths $n\in\{100,\allowbreak250,\allowbreak500,\allowbreak1000,\allowbreak2000\}$, starting from $x_{-2}=x_{-1}=x_0=0$.
For these time-series observations, the encoder combines $48$ DeepSets outputs computed from lag windows with $16$ LSTM outputs, giving a summary of dimension $q=64$.
For the comparison in Figure~\ref{fig:new-bh-energy}, we use three series generated at this fixed truth and three training seeds.
For Raw DGQB, we directly learn a partially input-convex conditional potential $\psi(u,X)$, allowing nonlinear dependence on $X$ while preserving convexity in $u$.

Since the posterior is unavailable in closed form, we evaluate the learned posterior draws against the fixed truth using the energy score \citep{gneiting2007strictly}.
For $M$ posterior draws $\theta^{(1)},\ldots,\theta^{(M)}$, its sample estimate is
\begin{equation}
\widehat{\operatorname{ES}}
=\frac{1}{M}\sum_{m=1}^{M}\norm{\theta^{(m)}-\theta_0}
-\frac{1}{2M(M-1)}\sum_{m\ne\ell}\norm{\theta^{(m)}-\theta^{(\ell)}},
\label{eq:new-energy-score}
\end{equation}
where $\norm{\cdot}$ is the Euclidean norm and we use $M=2000$.
The first term measures distance to the truth, while the second accounts for the spread of the posterior draws.
Lower scores indicate better performance under this criterion.
We first average scores over the three observed series for each training seed, then report the mean and standard deviation across all three seeds.

\begin{figure}[!t]
    \centering
    \includegraphics[width=0.48\linewidth]{\arxivdir 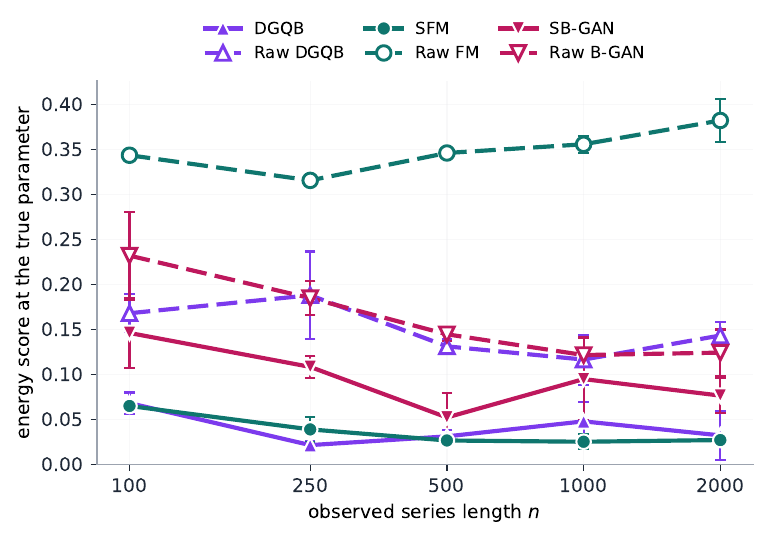}
    \caption{Brock--Hommes energy score: mean and SD across three seeds. Solid: summaries; dashed: raw data.
    }
    \label{fig:new-bh-energy}
\end{figure}

Figure~\ref{fig:new-bh-energy} shows that all three methods with jointly learned summaries achieve lower mean energy scores than their raw-data counterparts at every observation size considered.
The reductions are particularly pronounced for DGQB and SFM.
At $n=2000$, the mean scores are $0.032$, $0.027$, and $0.076$ for DGQB, SFM, and SB-GAN, compared with $0.143$, $0.382$, and $0.124$ for their raw-data counterparts.
Thus, jointly learned summaries improve inference performance under the energy score across different posterior learners in the Brock--Hommes model.
Jointly learned summaries also yield lower mean energy scores for NPE and NRE than their raw-data counterparts; see Appendix~\ref{app:new-bh-contours}.

\section{Discussion}
\label{sec:discussion}

This paper studies joint summary and posterior learning as a general principle for scalable Bayesian SBI.
The key mechanism is two-way guidance: a structured summary simplifies posterior learning, while a suitable posterior objective guides the summary toward posterior sufficiency.
We explain how structural restrictions on learned summaries can counterintuitively improve posterior approximation.
In practice, performance depends on tuning the summary dimension, network architecture, and training settings.
These choices affect both learning complexity and how much posterior information the summary retains.
An important direction for future research is to select the summary dimension adaptively according to the inference problem and available simulation budget.

\section*{Acknowledgements}
The authors are grateful to Veronika Ro\v{c}kov\'a for invaluable discussions and insightful suggestions that helped shape this work.

\clearpage
\appendix
\section{Methods and Implementation Details}
\label{app:original-methods}

We review the five methods used in this paper and describe how we incorporate learned summaries into their training and posterior sampling procedures.
We retain the notation $S=f_\eta(X)$ for the learned summary and $U$ for independent source noise.
The expectations in the training objectives average over the prior, simulator, and any auxiliary draws, with the candidate networks held fixed.

\subsection{Deep Generative Quantile Bayes (DGQB)}
\label{app:original-qr}

\citet{kim2025quantile} learn a conditional vector quantile map for posterior sampling.
In our notation, this map is $G(u,s)=\nabla_u\psi(u,s)$, where the potential $\psi$ is convex in $u$ for each fixed $s$.
The conditional vector quantile transports the reference distribution $F_U$ to $\pi(\cdot\mid S=s)$ \citep{carlier2016vqr}.
Their True Quantile Affinity assumption posits a sufficient summary for which the potential has the representation $\psi(u,f(X))=\varphi(u)+b(u)^\top f(X)$.
The functions $\varphi$ and $b$ are the functional coefficients of this representation.
Their Equation~(3.4), reproduced in \eqref{eq:qr-loss}, trains these functions and the summary network jointly, with batch normalization used to maintain centered summaries.

For the general formulation in Section~\ref{sec:qr}, we use the conditional Monge--Kantorovich dual objective, corresponding to Equation~(2.3) of \citet{kim2025quantile}.
For a reference domain $\mathcal U$, define the conjugate in $u$ and the population objective by
\begin{equation}
\begin{aligned}
\psi^*(\theta,s)&=\sup_{u\in\mathcal U}\{u^\top\theta-\psi(u,s)\},\\
\mathcal L_Q(\psi,f)&=\E_{U,X}\psi(U,f(X))+\E_{\theta,X}\psi^*(\theta,f(X)),
\qquad U\perp X.
\end{aligned}
\label{eq:qr-population}
\end{equation}
The empirical objective \eqref{eq:qnn-general-loss} averages the first term over all pairs of source draws and observations, and replaces the conjugate supremum by a maximum over the sampled source points.
The finite maximum approximates the source supremum; its error is considered in Appendix~\ref{app:v3-qnn-rate}.

For the affine potential in \eqref{eq:qr-potential}, centered batch summaries cancel the summary-dependent part of the first term, reducing that term to the average of $\varphi(U_i)$ and giving \eqref{eq:qr-loss}.
The functional coefficients determine the quantile map given the summary, while gradients of this same objective update the summary network through $b(U_j)^\top f_\eta(X_i)$.
The quantile representation requires the total potential $\varphi(u)+b(u)^\top s$ to be convex in $u$.

After training, we hold the encoder and its normalization fixed and compute $s^*=f_{\hat\eta}(X^*)$ once for the observed dataset.
Each fresh source draw $U$ then gives a parameter draw $\nabla_u\widehat\psi(U,s^*)$.
Algorithm~\ref{alg:qr} summarizes these steps.

Our nonlinear Raw DGQB baseline directly learns the conditional potential $\psi(u,X)$ from raw observations using the general objective in \eqref{eq:qnn-general-loss}.
We use a partially input-convex neural network \citep{amos2017input}, which allows nonlinear dependence on $X$ while preserving convexity in $u$ for every fixed observation.
This extends the affine specification of \citet{kim2025quantile} to a potential with nonlinear dependence on the raw observation.
The affine-in-$X$ specialization discussed in Section~\ref{sec:qr} is a separate choice of potential class.

\subsection{Summary flow matching (SFM)}
\label{app:original-fm}

\citet{zhai2026cfm} construct a dynamic block-triangular map from a source distribution to the joint distribution of observations and parameters.
At $t=0$, the observation source $X_0$ and parameter source $\theta_0=U$ are independent; at $t=1$, the target is the simulated joint law of $(X,\theta)$.
The corresponding block-triangular velocity field has the form
\begin{equation}
\dot X_t=v^X_\alpha(t,X_t),\qquad
\dot\theta_t=v^\theta_\beta(t,X_t,\theta_t).
\label{eq:raw-fm-flow}
\end{equation}
The observation velocity depends only on $(t,X_t)$, whereas the parameter velocity depends on $(t,X_t,\theta_t)$.
They fit this velocity field using the conditional flow matching loss.
For the straight-line interpolation paths in their Algorithm~1, the regression targets are $X_1-X_0$ and $\theta_1-\theta_0$.

Let $F_t(x_0)$ be the observation flow, so that $F_0(x_0)=x_0$.
For a well-posed invertible flow and a fixed observation $x$, first solve the observation ODE backward to obtain $F_1^{-1}(x)$.
Then draw $U$ and solve
\begin{equation}
\dot\theta_t=v^\theta_\beta\big(t,F_t(F_1^{-1}(x)),\theta_t\big),
\qquad \theta_0=U,
\qquad G(U,x)=\theta_1.
\label{eq:raw-fm-conditional}
\end{equation}
If the joint flow transports the product source to the correct target joint law, this procedure samples the conditional distribution of $\theta$ given $X=x$ \citep{baptista2024conditional}.
Note that the observation input follows its learned trajectory during sampling.

Our SFM modification replaces the observation block by $S=f_\eta(X)$ and uses the summary and parameter ODEs in \eqref{eq:sfm-flow}.
During training, independent source and target pairs give the straight-line interpolants in \eqref{eq:sfm-interpolants}.
The parameter velocity is regressed on $\theta_1-\theta_0$, and the summary velocity on $S_1-S_0$, using \eqref{eq:sfm-losses}.
The parameter loss updates the encoder and parameter velocity together through $S_t$.
We apply the stop-gradient operation to the summary input and regression target in the summary loss, which updates only the summary velocity.
Thus, the summary network is trained through the regression loss for the parameter velocity.
The resulting summary is reused across posterior draws: solve its ODE backward once, then integrate the joint flow forward for fresh parameter sources, as in Algorithm~\ref{alg:sfm}.
The conditions linking these regression objectives to the required conditional transport are discussed in Appendix~\ref{app:v3-sfm-calibration}.

\subsection{Summary Bayesian GAN (SB-GAN)}
\label{app:original-bgan}

B-GAN trains a conditional generator and a critic through a Wasserstein minimax objective, comparing simulated pairs $(X,\theta)$ with generated pairs $(X,g_\beta(U,X))$ \citep{wang2026generative}.
SB-GAN uses the affine-summary generator in \eqref{eq:sbgan-generator}, with the encoder and generator coefficients trained together.
The critic $D_\omega(X,\theta)$ receives the observation through its own network, with no parameters shared with the summary encoder.

Training alternates critic updates with joint encoder and generator updates.
For a minibatch of size $m$, write $\tilde\theta_i=g_\beta(U_i,f_\eta(X_i))$ and $\bar\theta_i=\epsilon_i\theta_i+(1-\epsilon_i)\tilde\theta_i$, where $\epsilon_i\sim\operatorname{Unif}[0,1]$ independently.
With penalty weight $\lambda_{\mathrm{gp}}>0$, the losses are
\begin{equation}
\begin{aligned}
\widehat{\mathcal L}_D
&=\frac1m\sum_{i=1}^m\{D_\omega(X_i,\tilde\theta_i)-D_\omega(X_i,\theta_i)\}
+\frac{\lambda_{\mathrm{gp}}}{m}\sum_{i=1}^m
\big[\max\{0,\norm{\nabla_\theta D_\omega(X_i,\bar\theta_i)}-1\}\big]^2,\\
\widehat{\mathcal L}_G
&=-\frac1m\sum_{i=1}^mD_\omega(X_i,\tilde\theta_i).
\end{aligned}
\label{eq:bgan-practical-losses}
\end{equation}
The one-sided gradient penalty penalizes parameter-gradient norms above one at the interpolated parameters, with the observation held fixed, following Appendix~E.1 of \citet{wang2026generative}.
Critic updates hold the encoder and generator fixed; generator updates hold the critic fixed and change both $\eta$ and $\beta$.
During the joint encoder and generator updates, the gradient of the critic score passes through the generator to the summary network.
The critic's input is $X$, so it is not restricted to the generator's learned summary $f_\eta(X)$.
The ideal conditional Wasserstein interpretation and its approximation by a neural critic are given in Appendix~\ref{app:v3-bgan-calibration}.

For posterior sampling, compute $s^*=f_{\hat\eta}(X^*)$ and evaluate the fitted generator at independent source draws, as in Algorithm~\ref{alg:sbgan}.
Our Raw B-GAN baseline supplies the observation directly to its conditional generator.
In the DeepSets construction of Section~2.5 of \citet{wang2026generative}, the pooled generator features depend on both the noise draw and the individual observations.
SB-GAN computes $f_\eta(X)$ from the observation alone and reuses this statistic across source draws.

\subsection{Neural posterior estimation (NPE)}
\label{app:original-npe}
\label{app:npe-nre}

Neural posterior estimation fits a conditional density estimator to simulated parameter--observation pairs \citep{greenberg2019automatic}.
\citet{greenberg2019automatic} develop automatic posterior transformation for sequential proposals; our experiments use the prior-trained NPE formulation described in their Section~2.2.
We draw parameters from the prior throughout training and simulate observations conditional on those parameters.
In our summary-based implementation, the conditional density estimator $\pi_\beta(\theta\mid f_\eta(X))$ is a masked autoregressive flow \citep{papamakarios2018maskedautoregressiveflowdensity}.
Here $\beta$ denotes the density-network parameters.
This is a composition of invertible autoregressive transformations conditioned on the summary, so its density can be evaluated by change of variables.
The summary network and conditional density estimator are trained together by minimizing the negative log-likelihood
\begin{equation}
\mathcal L_{\mathrm{NPE}}(\eta,\beta)
=-\E\log\pi_\beta\{\theta\mid f_\eta(X)\}.
\label{eq:app-npe-loss}
\end{equation}
The density estimator conditions on the learned summary, while backpropagation of the negative log-likelihood updates the summary network jointly with the density estimator.
Appendix~\ref{app:npe-population} gives the corresponding information-loss decomposition.

For posterior sampling, compute the observed summary once, draw Gaussian base variables, and map them to parameter draws through the fitted conditional flow.
Raw NPE conditions the density network on $X$.
The summary-based implementation permits nonlinear dependence of the density network on $S$.

\subsection{Neural ratio estimation (NRE)}
\label{app:original-nre}

\citet{hermans2020likelihoodfreemcmcamortizedapproximate} learn an amortized estimator of the likelihood-to-evidence ratio by distinguishing dependent parameter--observation pairs from independent pairs.
By Bayes' rule, this ratio also equals the posterior-to-prior density ratio.
The positive class consists of jointly simulated pairs $(\theta,X)$, and the negative class pairs an observation with an independently drawn prior parameter $\theta'$.
Let $D_\beta(\theta,S)$ be the probability assigned to the positive class, with equal class probabilities.
Here $\beta$ denotes the classifier parameters.
The classifier and summary encoder minimize the binary cross-entropy
\begin{equation}
\mathcal L_{\mathrm{NRE}}(\eta,\beta)
=-\frac12\E\log D_\beta\{\theta,f_\eta(X)\}
-\frac12\E\log\big[1-D_\beta\{\theta',f_\eta(X)\}\big].
\label{eq:app-nre-loss}
\end{equation}
In our implementation, each minibatch constructs negative pairs by a random nonzero cyclic shift of the parameters, so each observation is paired with a parameter from another simulated pair.
Both the classifier and encoder receive gradients from this objective.

For a fixed summary and wherever the density ratio is defined, the optimal classifier satisfies
\begin{equation}
\frac{D^*(\theta,S)}{1-D^*(\theta,S)}
=\frac{\pi(\theta\mid S)}{\pi(\theta)}.
\label{eq:app-nre-ratio}
\end{equation}
\citet{hermans2020likelihoodfreemcmcamortizedapproximate} use the estimated ratio in MCMC samplers.
In the Gaussian experiment, we use sampling--importance--resampling with the prior as the proposal distribution.
We draw a finite pool of parameter candidates from the prior, evaluate their fitted classifier odds at the observed summary, and resample with replacement using the normalized weights to obtain approximate posterior draws.
For Brock--Hommes, we use parallel Metropolis chains targeting the learned ratio times the known prior, as detailed in Appendix~\ref{app:new-bh-contours}.

The classifier uses the learned summary as its observation input, while gradients of the binary cross-entropy update that summary using the distinction between dependent and independent pairs.
The population justification is given in Appendix~\ref{app:nre-population}.
Raw NRE supplies $X$ to the classifier; the summary-based version allows nonlinear dependence on $S$.

\section{General Theory for Posterior Accuracy and Scalability}
\label{app:v3-theory-proofs}

This appendix proves the general posterior-risk guarantees and explains when joint learning recovers a sufficient summary.
We then quantify the learning cost of affine summary dependence and derive sufficient simulation budgets for recovering increasingly concentrated posteriors.
\subsection{Proof of Lemma~\ref{lem:v3-risk-decomposition}}
\label{app:v3-lemma-proof}

Conditional distributions are regular conditional probabilities on Euclidean spaces, and all kernels, fitted maps, and displayed suprema are measurable.
The evaluation pair $(X,\theta)$ and source noise are independent of training.
A posterior conditional on the fitted summary always holds the fitted function fixed.
The integrated second-moment assumptions in Section~\ref{sec:learning} ensure the finiteness of the transport costs below.
The learner $G$ also indexes its native representation and sampling rule when these are needed to specify the objective.

Fix a training realization for which the conditional second moments are integrable, and write $S=\widehat f(X)$.
For almost every $X$, the triangle inequality gives
\[
W_2\{\widehat\pi(\cdot\mid X),\pi(\cdot\mid X)\}
\leq W_2\{\pi(\cdot\mid S),\pi(\cdot\mid X)\}
+W_2\{\widehat\pi(\cdot\mid X),\pi(\cdot\mid S)\}.
\]
Squaring, applying $(a+b)^2\leq2a^2+2b^2$, and averaging over $X$ and training gives the first inequality.

For the reverse inequality, keep training fixed and let $Q_s=\Law\{\widehat G(U,s)\}$.
For each $x$, choose an optimal coupling of $\pi(\cdot\mid x)$ and $Q_{\widehat f(x)}$, measurably in $x$.
Such a choice exists for the continuous squared Euclidean cost; see \citet[Corollary~5.22]{villani2009optimal}.
Draw $X$ from its prior predictive distribution and $(\theta,Y)$ from this coupling conditional on $X$.
This construction preserves the original joint law of $(X,\theta)$ and satisfies
\[
\Law(Y\mid X)=Q_S,
\qquad
\E\norm{\theta-Y}^2
=\E_XW_2^2\{\pi(\cdot\mid X),Q_S\}.
\]
Here and until the final averaging step, $\E$ averages over the fresh observation and the auxiliary coupling variables, with training held fixed.
Since the conditional law of $Y$ given $X$ depends only on $S$, we have $Y\perp X\mid S$.

On an enlarged probability space, draw $\theta'$ from $\Law(\theta\mid S,Y)$, conditionally independently of $(X,\theta)$ given $(S,Y)$.
The preceding conditional independence and the tower property imply
\[
\Law(\theta'\mid X=x)
=\int\Law\{\theta\mid S=\widehat f(x),Y=y\}\,Q_{\widehat f(x)}(dy)
=\pi\{\cdot\mid S=\widehat f(x)\}.
\]
Thus $(\theta,\theta')$ conditional on $X$ is a coupling of the full posterior and summary posterior.
Also, $\theta$ and $\theta'$ are conditionally independent copies given $(S,Y)$, so
\[
\E\norm{\theta-\theta'}^2
=2\E\norm{\theta-\E(\theta\mid S,Y)}^2.
\]
The conditional-mean decomposition gives
\[
\E\norm{\theta-Y}^2
=\E\norm{\theta-\E(\theta\mid S,Y)}^2
+\E\norm{\E(\theta\mid S,Y)-Y}^2.
\]
Combining these identities yields
\[
\E_XW_2^2\{\pi(\cdot\mid X),\pi(\cdot\mid S)\}
\leq\E\norm{\theta-\theta'}^2
\leq2\E\norm{\theta-Y}^2
=2\E_XW_2^2\{\pi(\cdot\mid X),Q_S\}.
\]
Averaging over training proves $R_f\leq2R_p$.
This argument does not require a training objective or a calibration condition.

\subsection{Proof of Theorem~\ref{thm:v3-two-way}}
\label{app:v3-theorem-proof}

For this proof, write
\[
D_N=\sup_{f\in\mathcal F,\,G\in\mathcal G_f}
|\mathcal L_N(f,G)-\mathcal L(f,G)|.
\]
The classes are fixed before training, and the fitted pair belongs to them.
For every admissible comparison pair $(f,G)$, the definition of $\xi_N$ gives, for each training realization,
\[
\begin{aligned}
\mathcal L(\widehat f,\widehat G)-\mathcal L_0
&\leq\mathcal L_N(\widehat f,\widehat G)-\mathcal L_0+D_N\\
&\leq\mathcal L_N(f,G)-\mathcal L_0+D_N+\xi_N\\
&\leq\mathcal L(f,G)-\mathcal L_0+2D_N+\xi_N.
\end{aligned}
\]
Taking the infimum over $(f,G)$ therefore yields
\[
0\leq\mathcal L(\widehat f,\widehat G)-\mathcal L_0
\leq\Delta_{\mathcal F}+2D_N+\xi_N.
\]
The left inequality follows from the definition of the common full-data optimum.
Neither the population nor the empirical infimum needs to be attained.
The assumed integrability of $D_N$ and $\xi_N$ also makes the population excess integrable.

Uniform calibration applies to the fitted pair with training held fixed.
Taking the training expectation, applying Jensen's inequality to the concave function $\Gamma$, and using its monotonicity gives
\[
\begin{aligned}
R_p
&\leq\E_N\Gamma\{\mathcal L(\widehat f,\widehat G)-\mathcal L_0\}\\
&\leq\Gamma\!\left(\E_N\{\mathcal L(\widehat f,\widehat G)-\mathcal L_0\}\right)\\
&\leq\Gamma\big(\Delta_{\mathcal F}+2\E_ND_N+\E_N\xi_N\big).
\end{aligned}
\]
This proves \eqref{eq:v3-end-to-end}, and Lemma~\ref{lem:v3-risk-decomposition} gives $R_f\leq2R_p$.
For fixed $n$ and classes, $\Delta_{\mathcal F}=0$, $\E_ND_N\to0$, and $\E_N\xi_N\to0$ imply $R_p\to0$ by continuity of $\Gamma$ at zero, and hence $R_f\to0$.
Since squared Wasserstein distances are nonnegative, $R_f=0$ is equivalent to equality of the full and summary posteriors for almost every training realization and almost every fresh observation.

\subsection{Sufficient conditions for statistical and optimization consistency}
\label{app:v3-consistency}

We give sufficient conditions for the statistical and optimization errors in Theorem~\ref{thm:v3-two-way} to vanish.
The gap $\Delta_{\mathcal F}$ is fixed when the summary and posterior classes are fixed.
Its being zero requires approximation of the full-data optimum in the native objective, including compatibility between the summary and posterior learner.
If classes are enlarged with $N$, convergence of the resulting sequence $\Delta_{\mathcal F_N}$ requires a separate approximation argument.

For an ordinary empirical criterion, let $\mathcal L_N(f,G)=N^{-1}\sum_i\ell_{f,G}(Z_i)$ and $\mathcal L(f,G)=\E\ell_{f,G}(Z_1)$ for independent identically distributed records.
Fix the candidate class before training and suppose its losses have an integrable envelope and a finite admissible uniform $\delta$-cover for every $\delta>0$.
Such a cover approximates every loss by one representative within $\delta$ on every possible record.
For optimization, choose a deterministic sequence $\delta_N\downarrow0$ and select the first empirical minimizer over a finite uniform $\delta_N$-cover of losses, using a fixed enumeration of its admissible representatives.
Assume the displayed suprema are measurable.

\begin{claim}[Sufficient conditions for consistency]
\label{claim:b-consistency}
Under these conditions,
\[
\E_N\sup_{f,G}|\mathcal L_N(f,G)-\mathcal L(f,G)|\longrightarrow0,
\qquad
0\le\xi_N\le\delta_N.
\]
If $\Delta_{\mathcal F}=0$ and the calibration condition of Theorem~\ref{thm:v3-two-way} holds, then $R_p\to0$ and $R_f\to0$ for fixed $n$ as $N\to\infty$.
\end{claim}

This supplies concrete sufficient conditions for both learning errors in Theorem~\ref{thm:v3-two-way}.
The finite-set optimizer establishes attainability of the tolerance; practical neural training must establish its own optimization accuracy.

\begin{proof}
Fix a finite uniform $\delta$-cover with representative losses $\ell_1,\ldots,\ell_M$.
Write $P_N\ell=N^{-1}\sum_i\ell(Z_i)$ and $P\ell=\E\ell(Z_1)$.
For every training realization, the supremum deviation is at most $\max_j|(P_N-P)\ell_j|+2\delta$.
The integrable envelope ensures the law of large numbers in $L^1$ for each representative, so
\[
\E_N\max_{j\le M}|(P_N-P)\ell_j|
\le\sum_{j=1}^M\E_N|(P_N-P)\ell_j|\longrightarrow0.
\]
Take the limit superior as $N\to\infty$, then let $\delta\downarrow0$.

The empirical minimizer over the finite candidate set is measurable.
Compare it to a representative of an arbitrarily accurate empirical minimizer over the full class.
This gives
\[
\mathcal L_N(\widehat f,\widehat G)
\le\inf_{f,G}\mathcal L_N(f,G)+\delta_N,
\]
and hence $\xi_N\le\delta_N$.
Theorem~\ref{thm:v3-two-way} now gives both risk limits.
\end{proof}

Exact global empirical minimization also suffices, since it gives $\xi_N=0$.
For product averages and adversarial objectives, the method-specific bounds in Appendix~\ref{app:v3-calibration} supply the corresponding statistical controls.
For an adversarial objective, the optimization gap is measured after maximizing over the critic, so errors in that inner optimization must also be controlled.
An optimization error tending to zero at an unspecified rate suffices for fixed-$n$ consistency but cannot be omitted from a finite-sample or growing-$n$ bound.
Corollary~\ref{cor:v3-affine-rate} requires the optimization error to be no larger than its displayed statistical rate.

\subsection{Learning complexity and finite-sample bounds}
\label{app:v3-affine-proof}

\subsubsection{Function class complexity under affinity}
\label{app:d-affine-complexity}

For the complexity comparison, we consider affine summary dependence as a concrete structural choice.
This also includes dependence that is affine in a fixed finite set of features of a summary, by absorbing those features into the encoder, provided the resulting class satisfies the conditions below.
We show how affinity reduces the complexity of posterior learning while accounting for the cost of learning the summary.

Work on bounded domains rescaled to $X\in[-1,1]^{nd_X}$ and $w\in[-1,1]^p$, retaining the rescaling in the regularity bounds.
Here $n$ is the observation size, $d_X$ is the dimension of each observation, and $p\ge1$ is the dimension of the arguments other than the observation or summary.
The affine class has the form
\[
h(w,X)=a(w)+B(w)f_\zeta(X)
=A(w)(1,f_\zeta(X)^\top)^\top,
\qquad A(w)=[a(w),B(w)].
\]
The function $h$ takes values in $\R^r$, and the learned summary $f_\zeta(X)$ takes values in $[-1,1]^q$.
The raw class instead consists of functions $h(w,X)$ of all $nd_X+p$ inputs.

Let $P_n$ count the encoder and auxiliary-network parameters, collected in $\zeta$.
The auxiliary networks include the critic for SB-GAN and the summary-velocity branch for SFM when these enter the learning objective.
For the raw class, $\zeta$ contains only auxiliary parameters, and $P_n=0$ when none are present.
All candidate classes are fixed before training.

For each class, use a common regularity bound $L_n\ge1$ and restrict $\zeta$ to $[-L_n,L_n]^{P_n}$.
For a matrix $A(w)$, write $\|A(w)\|_{\infty,1}=\max_{j\le r}\sum_{k=0}^q|A_{jk}(w)|$ for its maximum absolute row sum; the corresponding norm for a matrix-valued function also takes the supremum over $w$.
For the affine class, assume
\[
\sup_w\|A(w)\|_{\infty,1}\le L_n,
\qquad
\|A(w)-A(w')\|_{\infty,1}\le L_n\|w-w'\|_\infty.
\]
Let $\ell_{\zeta,A}$ denote the complete normalized loss contribution, including any auxiliary inputs, and assume these losses lie in a common interval of length one.
Require the loss stability bound
\begin{equation}
\|\ell_{\zeta,A}-\ell_{\zeta',A'}\|_\infty
\le L_n\big\{\|A-A'\|_{\infty,1}+\|\zeta-\zeta'\|_\infty\big\},
\label{eq:v3-native-loss-cover-transfer}
\end{equation}
where the norm on losses is uniform over all possible records.
Thus $L_n$ includes the sensitivity of the complete objective to the encoder and auxiliary parameters, including their effects on targets, interpolation states, or critic evaluations.
For the raw class, assume that each component of $h$ is bounded by $L_n$ and is $L_n$-Lipschitz in $(w,X)$ under the input maximum norm.
Require the same loss stability bound with $A$ replaced by $h$ and the matrix norm replaced by the componentwise uniform norm.
The two classes may have different $L_n$ and $P_n$.
For DGQB, also require the same stability bound for the normalized potentials defined in Appendix~\ref{app:v3-qnn-rate}.

For $\delta\in(0,1]$, let $M_n(\delta)$ be the smallest number of admissible representative losses needed to approximate every loss in the class uniformly within $\delta$.
In other words, every candidate loss differs from at least one representative by at most $\delta$ on every possible record.
This count measures the complexity of the complete loss class at accuracy $\delta$, including the jointly learned summary and auxiliary networks.
For the DGQB product criterion, use the corresponding count of normalized-potential representatives; conjugate stability transfers this cover to the objective as shown in Appendix~\ref{app:v3-qnn-rate}.
Define $A_n=32r(q+1)L_n^2$ for the affine class and $A_n=32rL_n^2$ for the raw class; any larger values also suffice.
These constants retain the output and summary dimensions and the regularity bounds, including the effects of rescaling.
Put $s=p$ for the affine class and $s=nd_X+p$ for the raw class.

\begin{proposition}[Loss complexity under affinity]
\label{prop:d-affine-complexity}
Under these conditions, for every $\delta\in(0,1]$,
\begin{equation}
\log M_n(\delta)
\le\left\{P_n+\left(\frac{A_n}{\delta}\right)^s\right\}
\log\left(\frac{A_n}{\delta}\right).
\label{eq:v3-complete-entropy}
\end{equation}
\end{proposition}

The two terms account for learning the finite-dimensional networks and learning the coefficient functions, respectively, with the latter replaced by the full function $h$ in the raw class.
Affinity reduces the dimension governing the function approximation from $nd_X+p$ to $p$.
Note that the summary dimension $q$ and the encoder's learning cost remain in $A_n$ and $P_n$.
Thus the bound explains the complexity reduction while retaining the cost of obtaining the summary.
The comparison concerns upper bounds for these specified classes; further structure in a raw-data architecture can reduce its complexity.

\begin{proof}
First cover the finite-dimensional parameters.
Partition $[-L_n,L_n]^{P_n}$ into cells of maximum-norm diameter at most $\delta/(2L_n)$.
There are at most $(1+4L_n^2/\delta)^{P_n}$ cells, and parameters in the same cell differ by at most $\delta/(2L_n)$.

Next cover the coefficient functions of the affine class.
Divide each coordinate interval of $[-1,1]^p$ into $\lceil16L_n^2/\delta\rceil$ equal parts and use the resulting tensor grid, including its endpoints.
Its mesh is at most $\delta/(8L_n^2)$.
The grid has at most $(2+16L_n^2/\delta)^p$ points.
At each grid point, quantize each of the $r(q+1)$ entries of $A$ into bins of width $\delta/\{4L_n(q+1)\}$.
Each entry has at most $1+8(q+1)L_n^2/\delta$ bins.
Two coefficient functions with the same code have row-sum discrepancy at most $\delta/(4L_n)$ at the grid points.
Their two Lipschitz errors between an input and its nearest grid point add at most $\delta/(8L_n)$, so their uniform row-sum discrepancy is at most $\delta/(2L_n)$.

For each nonempty combination of parameter cell and coefficient code, select one complete admissible candidate as representative.
This selection preserves any convexity, flow, or critic restrictions that couple the candidate's components.
Equation~\eqref{eq:v3-native-loss-cover-transfer} then bounds the loss discrepancy within each combination by $\delta$.
The additional DGQB stability condition gives the same bound for its normalized-potential discrepancy.
Counting these combinations gives
\[
\log M_n(\delta)\le
P_n\log\left(1+\frac{4L_n^2}{\delta}\right)
+r(q+1)\left(2+\frac{16L_n^2}{\delta}\right)^p
\log\left(1+\frac{8(q+1)L_n^2}{\delta}\right).
\]
The stated choice of $A_n$, together with $\delta\le1$ and $p\ge1$, yields \eqref{eq:v3-complete-entropy} for $s=p$.

For the raw class, apply the same grid construction to the $r$ component functions of $h$ on $[-1,1]^{nd_X+p}$.
The input dimension becomes $s=nd_X+p$, and there are $r$ output values to quantize at each grid point.
The count therefore has $r$ in place of $r(q+1)$ and $s$ in place of $p$, with the same parameter-cell construction for any auxiliary networks.
Using this class's $A_n=32rL_n^2$ proves \eqref{eq:v3-complete-entropy} for the raw class.
\end{proof}

\subsubsection{Finite-dimensional networks}

We also give a cover for a fixed network architecture, without imposing affine dependence.
This covers the jointly trained density and classifier networks used by NPE and NRE.
Let $W_n$ count all trainable parameters of the encoder and posterior learner, collected in $\zeta$.
Unlike $P_n$ above, $W_n$ includes the entire posterior network.
Assume $\zeta\in[-L_n,L_n]^{W_n}$ for $L_n\ge1$ and that the complete normalized loss satisfies
\[
\|\ell_\zeta-\ell_{\zeta'}\|_\infty
\le L_n\|\zeta-\zeta'\|_\infty.
\]
Here the loss norm is uniform over all possible records, including all inputs to the encoder and posterior learner.
For this class, put $A_n=1+2L_n^2$.
Let $M_n(\delta)$ be the smallest number of admissible loss representatives needed to approximate every candidate loss within $\delta$, uniformly over these records.

\begin{claim}[Loss complexity of finite-dimensional networks]
\label{claim:b-parametric-cover}
Under these conditions, for every $\delta\in(0,1]$,
\[
\log M_n(\delta)\le W_n\log(A_n/\delta).
\]
\end{claim}

This bound counts the whole jointly trained architecture.
Its dependence on observation and summary dimensions remains in $W_n$ and the uniform sensitivity bound entering $A_n$.

\begin{proof}
Partition the parameter cube into cells of maximum-norm diameter at most $\delta/L_n$.
There are at most $(1+2L_n^2/\delta)^{W_n}$ cells.
Select one admissible parameter vector from each nonempty cell.
The loss stability bound makes these losses a uniform $\delta$-cover, and $\delta\le1$ gives the displayed count.
\end{proof}

\subsubsection{Finite sample error}

We next convert a covering bound into uniform error of the empirical objective.
Write $\overline{\mathcal L}_N(\ell)$ and $\overline{\mathcal L}(\ell)$ for the normalized empirical and population criteria indexed by a candidate function $\ell$.
For ordinary averages, $\ell$ is the complete loss contribution; for DGQB, it is the normalized potential whose distance controls the product criterion.
Fix the candidate class before training, and require both criteria to be 1-Lipschitz in the uniform distance between their indexing functions.
Thus a uniform $\delta$-cover with $M_n(\delta)$ admissible representatives approximates every empirical and population criterion value within $\delta$.
Assume all suprema below are measurable.
For a constant $c>0$, require, for every fixed candidate and $\lambda\in\R$,
\begin{equation}
\E_N\exp\!\left[\lambda\{\overline{\mathcal L}_N(\ell)-\overline{\mathcal L}(\ell)\}\right]
\le\exp\{c^2\lambda^2/(2N)\}.
\label{eq:b-reference-mgf}
\end{equation}
This condition concerns deviation from the population criterion itself.
For independent averages of losses in a common interval of length one, Hoeffding's lemma gives $c=1/2$.

\begin{claim}[Uniform empirical error from a finite cover]
\label{claim:b-uniform-deviation}
Under these conditions, for every $\delta\in(0,1]$,
\begin{equation}
\E_N\sup_\ell|\overline{\mathcal L}_N(\ell)-\overline{\mathcal L}(\ell)|
\le2\delta+c\sqrt{\frac{2\log\{2M_n(\delta)\}}{N}}.
\label{eq:v3-reference-deviation}
\end{equation}
\end{claim}

The bound separates approximation by finitely many representatives from statistical error over those representatives.
The constant $c$ retains the sampling structure of each method, including dependence between terms of an empirical objective.

\begin{proof}
For the $M_n(\delta)$ representatives $\ell_j$, apply \eqref{eq:b-reference-mgf} to both signs.
For $\lambda>0$, the exponential maximum bound gives
\[
\E_N\max_j|\overline{\mathcal L}_N(\ell_j)-\overline{\mathcal L}(\ell_j)|
\le\frac{\log\{2M_n(\delta)\}}{\lambda}
+\frac{c^2\lambda}{2N}.
\]
Minimizing over $\lambda$ gives the square-root term.
Approximating each candidate in its empirical and population criterion adds $2\delta$.
\end{proof}

For an adversarial objective, apply this claim to the loss before taking the critic supremum.
Indeed,
\[
\sup_{f,G}\left|\sup_D\overline{\mathcal L}_N(f,G,D)
-\sup_D\overline{\mathcal L}(f,G,D)\right|
\le\sup_{f,G,D}|\overline{\mathcal L}_N(f,G,D)-\overline{\mathcal L}(f,G,D)|.
\]
The critic belongs to the covered class; no unbiasedness of the maximized empirical objective is required.

To specialize to Proposition~\ref{prop:d-affine-complexity}, suppose also that the normalized reference discrepancy is uniformly at most one.
This holds when the empirical and population reference criteria lie in the same interval of length one, including after a critic supremum.

\begin{claim}[Explicit rate under the affine and raw class bounds]
\label{claim:b-cover-rate}
Under Claim~\ref{claim:b-uniform-deviation} with $c\le3/4$, the cover bound \eqref{eq:v3-complete-entropy}, and the bounded-discrepancy condition above, for $N\ge2$,
\[
\E_N\sup_\ell|\overline{\mathcal L}_N(\ell)-\overline{\mathcal L}(\ell)|
\le\frac72\left[
\sqrt{\frac{(P_n+1)\log(2A_nN)}{N}}
+\left\{\frac{A_n^s\log(2A_nN)}{N}\right\}^{1/(s+2)}
\right].
\]
If $c\le1/2$, the factor $7/2$ can be replaced by $3$.
\end{claim}

The exponent uses $s=p$ for affine summary dependence and $s=nd_X+p$ for the raw class, with each class's own $A_n$ and $P_n$.
The rate includes the cost of learning the summary and any auxiliary networks.

\begin{proof}
Choose $\delta=\{A_n^s\log(2A_nN)/N\}^{1/(s+2)}$ when this is at most one.
Since $A_n\ge e$ and $s\ge1$, we have $\delta\ge N^{-1/(s+2)}$ and $\log(A_n/\delta)\le\log(2A_nN)$.
The covering bound gives
\[
\log\{2M_n(\delta)\}
\le\{P_n+1+A_n^s\delta^{-s}\}\log(2A_nN).
\]
Substituting in Claim~\ref{claim:b-uniform-deviation} and separating the square root yields
\[
\E_N\sup_\ell|\overline{\mathcal L}_N-\overline{\mathcal L}|
\le\sqrt2c\sqrt{\frac{(P_n+1)\log(2A_nN)}{N}}
+\left(2+\sqrt2c\right)\delta.
\]
For $c\le3/4$, both coefficients are at most $7/2$; for $c\le1/2$, they are at most $3$.
If the proposed resolution exceeds one, the bound of one on the reference discrepancy proves both statements directly.
\end{proof}

\subsection{Proof of Corollary~\ref{cor:v3-affine-rate}}
\label{app:v3-corollary-proof}

We combine the preceding empirical bound with the general posterior guarantee.
Let $s=p$ for the affine class and $s=nd_X+p$ for the raw class, using each class's own $P_n$ and $A_n$.
For this proof and the simulation budgets below, write
\[
b_N=\sqrt{\frac{(P_n+1)\log(2A_nN)}{N}}
+\left\{\frac{A_n^s\log(2A_nN)}{N}\right\}^{1/(s+2)}.
\]
This abbreviates the sum already displayed in \eqref{eq:v3-affine-rate} and its raw counterpart.

Let $\overline{\mathcal L}_N$ and $\overline{\mathcal L}$ be the reference empirical and population criteria before source discretization, critic approximation, or numerical errors are included.
For independent-record losses in a common interval of length one, the hypotheses of Claim~\ref{claim:b-uniform-deviation} hold with $c=1/2$.
For an adversarial objective, apply that claim before the critic supremum and use the supremum inequality proved there.
The DGQB product criterion has $c=1/\sqrt8$ under Claim~\ref{claim:c-dgqb-empirical}; this uses its two independent samples and does not treat their $N^2$ combinations as independent.
For NRE, the dependent negative-pair construction gives $c=3/4$ by the bounded-differences calculation in Claim~\ref{claim:c-nre-empirical}.
The corollary assumes these native reference-criterion conditions and the complete class cover; applying it to a particular method requires that method to satisfy both.
Claim~\ref{claim:b-cover-rate} therefore gives
\[
\E_N\sup_{f,G}|\overline{\mathcal L}_N(f,G)-\overline{\mathcal L}(f,G)|\le\frac72b_N.
\]

All criteria, approximation errors, and optimization gaps use the same deterministic normalization.
If the original loss scale is $V_n>0$, the normalized calibration is $\Gamma_n(t)=\Gamma_n^{\rm original}(V_nt)$; a common additive constant cancels from the excess.
The corollary's total objective-approximation condition is
\begin{equation}
\E_N\sup_{f,G}|\mathcal L_N(f,G)-\overline{\mathcal L}_N(f,G)|
+\sup_{f,G}|\overline{\mathcal L}(f,G)-\mathcal L(f,G)|
\le b_N.
\label{eq:v3-auxiliary-criterion-error}
\end{equation}
This single allowance covers all source, critic, and numerical approximations together.
The triangle inequality gives
\[
\E_N\sup_{f,G}|\mathcal L_N(f,G)-\mathcal L(f,G)|\le\frac92b_N.
\]
With $\Delta_{\mathcal F}=0$ and $\E_N\xi_N\le b_N$, the empirical comparison in the proof of Theorem~\ref{thm:v3-two-way} yields
\[
\E_N\{\mathcal L(\widehat f,\widehat G)-\mathcal L_0\}
\le2\left(\frac92b_N\right)+b_N=10b_N.
\]
Conditional calibration, Jensen's inequality, and Lemma~\ref{lem:v3-risk-decomposition} now give
\[
R_p\le\Gamma_n(10b_N),
\qquad R_f\le2R_p.
\]
Substituting $b_N$ proves \eqref{eq:v3-affine-rate} and its raw-data counterpart.

\subsection{Posterior-scale accuracy and simulation budgets}
\label{app:v3-budget}

We now use Corollary~\ref{cor:v3-affine-rate} to determine sufficient simulation budgets for recovering increasingly concentrated posteriors.
For posterior uncertainty of order $n^{-1/2}$, we require $nR_p\to0$ and $nR_f\to0$.
Assume the corollary's conditions at each observation size $n$, and suppose
\[
A_n\lesssim n^a,\qquad P_n+1\lesssim n^b,\qquad
\Gamma_n(t)\lesssim n^ct^\alpha\quad(0\le t\le1),
\]
where $a,b,c\ge0$ and $0<\alpha\le1$ are fixed and the implicit constants are uniform in $n$ and $t$.
Keep $p\ge1$ and $d_X>0$ fixed.

\begin{proposition}[Sufficient budgets for posterior scalability]
\label{prop:b-scaling-budgets}
Under these conditions, affine summary dependence with
\[
N(n)=\lceil n^\nu\rceil,
\qquad
\nu>\max\left\{b+\frac{2(1+c)}\alpha,
\ ap+\frac{(p+2)(1+c)}\alpha\right\}
\]
gives $nR_p\to0$ and $nR_f\to0$.
For the raw class, the same limits hold with
\[
N(n)=\lceil n^{\nu n}\rceil,
\qquad
\nu>d_X\left\{a+\frac{1+c}\alpha\right\}.
\]
\end{proposition}

These budgets control approximation error relative to the posterior's shrinking uncertainty.
The comparison retains the encoder and auxiliary-network costs through $P_n$ and $A_n$, and the posterior calibration through $\Gamma_n$.
The budgets are sufficient for the specified classes; the result does not assert that the raw-data budget is necessary.

\begin{proof}
Use $b_N$ defined in Appendix~\ref{app:v3-corollary-proof}.
For affine conditioning, $s=p$ and $N(n)=\lceil n^\nu\rceil$ give
\[
b_N=O\!\left(n^{(b-\nu)/2}\sqrt{\log n}
+n^{(ap-\nu)/(p+2)}(\log n)^{1/(p+2)}\right).
\]
The stated strict inequalities imply $b_N=o\{n^{-(1+c)/\alpha}\}$.
Consequently $10b_N\le1$ eventually and
\[
nR_p\lesssim n^{1+c}(10b_N)^\alpha\longrightarrow0.
\]

For raw conditioning, $s=nd_X+p$ and $N(n)=\lceil n^{\nu n}\rceil$ make the finite-dimensional term smaller than every negative power of $n$.
For the remaining term,
\[
\log\left\{\left(\frac{A_n^s\log(2A_nN)}{N}\right)^{1/(s+2)}\right\}
\le\left(a-\frac\nu{d_X}\right)\log n+O(1).
\]
Indeed, $s/(s+2)=1+O(n^{-1})$, $n/(s+2)=d_X^{-1}+O(n^{-1})$, and $\log\log(2A_nN)=O(\log n)$.
The stated condition again gives $b_N=o\{n^{-(1+c)/\alpha}\}$ and hence $nR_p\to0$.
Finally, $R_f\le2R_p$ gives the summary-risk limit in both cases.
\end{proof}

The corollary's optimization and objective-approximation conditions must hold at the displayed statistical rate for each $n$.
They remain part of the scalability conditions.

We also give an explicit budget for a general posterior uncertainty scale $r_n>0$.
Choose $0<\tau_n\le1$ with $\Gamma_n(\tau_n)=o(r_n^2)$, and define
\[
B_{n,s}=\frac{400(P_n+1)}{\tau_n^2}
+\frac{20^{s+2}A_n^s}{\tau_n^{s+2}}.
\]
This quantity collects the budgets for the two terms in the corollary.
Use $s=p$ for affine dependence or $s=nd_X+p$ for the raw class, with the corresponding class constants.

\begin{claim}[An explicit budget for a general posterior scale]
\label{claim:b-general-budget}
Under Corollary~\ref{cor:v3-affine-rate}, the choice
\begin{equation}
N(n)\ge\left\lceil2B_{n,s}\log(4A_nB_{n,s})\right\rceil
\label{eq:v3-general-simulation-budget}
\end{equation}
gives
\[
R_p\le\Gamma_n(\tau_n)=o(r_n^2),
\qquad R_f\le2\Gamma_n(\tau_n)=o(r_n^2).
\]
\end{claim}

The right-hand side of \eqref{eq:v3-general-simulation-budget} depends only on the desired accuracy and class bounds, so it specifies a simulation budget without $N$ appearing implicitly inside a logarithm.

\begin{proof}
For this calculation put $B=B_{n,s}\ge1$ and $L_B=\log(4A_nB)>1$.
At $N_0=2BL_B$, we have $\log(2A_nN_0)=L_B+\log L_B\le2L_B$.
The function $N/\log(2A_nN)$ is increasing for $N\ge1$, since $A_n\ge e$.
Thus every integer $N\ge\lceil N_0\rceil$ satisfies $N/\log(2A_nN)\ge B$, and
\[
b_N\le\sqrt{\frac{P_n+1}{B}}
+\left(\frac{A_n^s}{B}\right)^{1/(s+2)}
\le\frac{\tau_n}{20}+\frac{\tau_n}{20}.
\]
Corollary~\ref{cor:v3-affine-rate} now gives $R_p\le\Gamma_n(10b_N)\le\Gamma_n(\tau_n)$ and $R_f\le2R_p$.
\end{proof}

For example, if $\Gamma_n(t)\le K_nt^\alpha$ with $K_n>0$ and $\rho_n\downarrow0$, then
\[
\tau_n=\min\left\{1,\left(\frac{\rho_nr_n^2}{K_n}\right)^{1/\alpha}\right\}
\]
satisfies $\Gamma_n(\tau_n)\le\rho_nr_n^2$.
These explicit budgets use class bounds specified as functions of $n$; bounds that also grow with $N$ require deterministic envelopes before the budget is applied.

\section{Method-Specific Guarantees and Two-Way Guidance}
\label{app:v3-calibration}
\label{app:v3-native-rates}

We examine how the general theory applies to the five SBI methods and explain how two-way guidance works through their different objectives.
To apply Theorem~\ref{thm:v3-two-way}, its conditions must hold uniformly over the corresponding candidate class.

\subsection{Deep Generative Quantile Bayes}
\label{app:v3-qnn-calibration}

DGQB learns a conditional vector quantile map through the conditional Monge--Kantorovich dual objective \citep{kim2025quantile}.
We first show how this objective guides summary learning, then relate its excess loss to posterior error, and finally control the finite-source approximation used in training.

\subsubsection{Two-way guidance}

We examine how two-way guidance works in DGQB by separating the population objective into the cost of losing posterior information and the cost of fitting the posterior given the summary.

Fix a measurable encoder $f$ and write $S=f(X)$, with the observation and summary spaces standard Borel.
Assume $U\sim F_U$ is independent of $X$, $F_U$ is absolutely continuous on $\R^d$, and $\E\norm U^2+\E\norm\theta^2<\infty$.
The source domain $\mathcal U$ is either $\R^d$ or the closed unit ball $K$, with $F_U(K)=1$ in the latter case.
For a potential $\psi(u,s)$ convex in $u$ and differentiable $F_U$-almost everywhere, define
\[
\begin{aligned}
G(u,s)&=\nabla_u\psi(u,s),
&\psi^*(y,s)&=\sup_{u\in\mathcal U}\{u^\top y-\psi(u,s)\},\\
\mathcal L(f,G)&=\E_{U,X}\psi(U,S)+\E_{\theta,X}\psi^*(\theta,S),
&\mathcal L_f^*&=\inf_{\psi}\mathcal L(f,\nabla_u\psi).
\end{aligned}
\]
The infimum is over all admissible measurable conditional convex potentials; their two dual terms are absolutely integrable.
Let $Q_X$ and $Q_S$ be the quadratic-cost optimal transport maps from $F_U$ to $\pi(\cdot\mid X)$ and $\pi(\cdot\mid S)$, respectively, and set $\mathcal L_0=\E_{U,X}\{U^\top Q_X(U)\}$.
Assume jointly measurable versions of these maps and attained measurable conditional dual optimizers whose terms are absolutely integrable after additive normalization.
Conditional Kantorovich duality identifies $\mathcal L_0$ as the full-data optimum and gives $\mathcal L_f^*=\E_{U,X}\{U^\top Q_S(U)\}$ \citep{carlier2016vqr,kim2025quantile}.

\begin{proposition}[Two-way guidance in DGQB]
\label{prop:c-dgqb-information}
Under the preceding conditions,
\[
\mathcal L(f,G)-\mathcal L_0
=\underbrace{\mathcal L(f,G)-\mathcal L_f^*}_{\text{potential-fitting gap}}
+\underbrace{\mathcal L_f^*-\mathcal L_0}_{\text{information-loss gap}},
\]
where both gaps are nonnegative, and
\[
\mathcal L_f^*-\mathcal L_0=0
\quad\Longleftrightarrow\quad
\pi(\cdot\mid X)=\pi(\cdot\mid f(X))\quad\text{almost surely}.
\]
\end{proposition}

This is the core of two-way guidance in DGQB.
A posterior-sufficient summary allows the potential to depend on $X$ through $S$ without losing posterior information.
In the reverse direction, an insufficient summary incurs a positive population penalty that cannot be eliminated by improving the potential while keeping the summary fixed.
A structural restriction such as affinity determines the admissible potentials and can add to the potential-fitting gap; the information-loss penalty remains.

\begin{proof}
Since $(S,Q_S(U))$ has the same law as $(S,\theta)$, Fenchel's inequality gives
\[
\mathcal L(f,G)-\mathcal L_f^*
=\E_{U,X}\big[\psi(U,S)+\psi^*(Q_S(U),S)-U^\top Q_S(U)\big]\ge0.
\]
Conditional on $S=s$, the law of $(U,Q_X(U))$ is a coupling of $F_U$ and $\pi(\cdot\mid S=s)$.
Its expected inner product cannot exceed that of the optimal coupling $(U,Q_s(U))$, so $\mathcal L_f^*\ge\mathcal L_0$.
Equivalently, expanding the squared transport cost yields
\[
\mathcal L_f^*-\mathcal L_0
=\frac12\left[\E_XW_2^2\{F_U,\pi(\cdot\mid X)\}
-\E_XW_2^2\{F_U,\pi(\cdot\mid S)\}\right].
\]

If the information-loss gap is zero, the conditional mixture coupling is optimal for almost every $s$.
Absolute continuity of $F_U$ makes this optimal coupling unique and concentrated on the graph of $Q_s$.
Thus $Q_X(U)=Q_S(U)$ almost surely, and Fubini's theorem gives equality of their pushforward distributions for almost every $X$.
Conversely, posterior sufficiency makes the unique optimal maps agree almost surely and therefore makes the information-loss gap zero.
\end{proof}

\subsubsection{Posterior accuracy}

We now show that an objective close to its full-data optimum produces a posterior close to the true posterior in $W_2$.
The conjugate domain matters for this step, so we first treat a potential defined on $\R^d$ and then the unit-ball source used by \citet{kim2025quantile}.

Write $z=\mathcal L(f,G)-\mathcal L_0$ and retain $\pi_{f,G}(\cdot\mid X)=\Law\{G(U,f(X))\mid X\}$ from Section~\ref{sec:learning}.
For the full-space case, take $\mathcal U=\R^d$ and assume the total potential is convex, continuously differentiable, and globally $L$-smooth in $u$, uniformly in $s$, for $L>0$.

\begin{claim}[Calibration for the full-space conjugate]
\label{claim:c-dgqb-global}
Under these conditions,
\[
\E_XW_2^2\{\pi_{f,G}(\cdot\mid X),\pi(\cdot\mid X)\}\le2Lz.
\]
\end{claim}

The objective therefore controls posterior error with $\Gamma(z)=2Lz$ in Theorem~\ref{thm:v3-two-way}, without assuming the candidate summary is sufficient.
The smoothness condition concerns the total potential, including $b(u)^\top S$ under affinity.

\begin{proof}
For $e=y-\nabla_u\psi(u,s)$, evaluate the conjugate at $u+e/L$ and apply the smoothness bound to obtain
\[
\psi(u,s)+\psi^*(y,s)-u^\top y\ge\frac{1}{2L}\norm{\nabla_u\psi(u,s)-y}^2.
\]
Since $(X,Q_X(U))$ has the same law as $(X,\theta)$, the expected left side at $y=Q_X(U)$ equals $z$.
Coupling $G(U,S)$ and $Q_X(U)$ through the same source draw then gives
\[
\E_XW_2^2\{\pi_{f,G}(\cdot\mid X),\pi(\cdot\mid X)\}
\le\E_{U,X}\norm{G(U,S)-Q_X(U)}^2\le2Lz.
\]
\end{proof}

For the unit-ball case, use the restricted conjugate $\psi_K^*(y,s)=\sup_{v\in K}\{y^\top v-\psi(v,s)\}$ throughout the objective defining $z$.
Assume $\norm\theta\le T$ almost surely and the total potential is continuously differentiable on a neighborhood of $K$, convex and $L$-smooth on $K$, with $\sup_{u\in K,s}\norm{\nabla_u\psi(u,s)}\le B$.
Put $M=T+B>0$ and suppose $\Pr(\norm U>1-a)\le c_Ua$ for $0<a\le1$, where $c_U\ge1$.
The radius-uniform source of \citet{kim2025quantile} has $c_U=1$; a volume-uniform source on $K$ permits $c_U=d$.

\begin{claim}[Calibration for the unit-ball conjugate]
\label{claim:c-dgqb-ball}
Under these conditions,
\[
\E_XW_2^2\{\pi_{f,G}(\cdot\mid X),\pi(\cdot\mid X)\}
\le\min\{M^2,\ 2Lz+2\sqrt{2c_UM^3z}\}.
\]
\end{claim}

The right side gives the calibration function $\Gamma$ for the bounded-source objective used in DGQB.
The additional square-root term accounts for source points near the boundary, where the full-space proof may evaluate the conjugate outside its domain.
This function is continuous, nondecreasing, and concave, and vanishes at zero, as required by Theorem~\ref{thm:v3-two-way}.

\begin{proof}
Fix $0<a\le1$ and $\norm u\le1-a$.
For $\norm y\le T$ and $e=y-\nabla_u\psi(u,s)\ne0$, the step $h=\min\{1/L,a/\norm e\}$ keeps $u+he$ in $K$.
Evaluating the restricted conjugate there gives
\[
\psi(u,s)+\psi_K^*(y,s)-u^\top y
\ge\frac h2\norm e^2
\ge\frac{\norm e^2}{2\max\{L,M/a\}}.
\]
The inequality also holds when $e=0$.
Apply it with $y=Q_X(U)$, and bound the squared error on the remaining boundary shell by $M^2$.
The expected Fenchel gap is $z$, so
\[
\E_{U,X}\norm{G(U,S)-Q_X(U)}^2
\le2Lz+\frac{2Mz}{a}+c_UM^2a.
\]
For $z>0$, choose $a=\sqrt{2z/(c_UM)}$ when this is at most one; otherwise the uniform bound $M^2$ already gives the claimed inequality.
For $z=0$, let $a\downarrow0$.
The same-source coupling completes the proof.
\end{proof}

\subsubsection{Finite sample error}
\label{app:v3-qnn-rate}

The training algorithm replaces the conjugate supremum by a maximum over sampled source points.
We control this approximation together with the statistical error, and use Proposition~\ref{prop:d-affine-complexity} to make the dependence on simulation budget and function class explicit.

Use the unit-ball setting and uniform bounds of Claim~\ref{claim:c-dgqb-ball}, and divide all objectives by $V_n=4M$, retaining their symbols.
This normalization changes the calibration to $\Gamma_n(t)=\Gamma(V_nt)$.
Take $N\ge2$ independent prior--simulator pairs and an independent bank of $N$ independent source draws in \eqref{eq:qnn-general-loss}.
Let $\mathcal L_N$ be the resulting normalized sample-maximum criterion and $\widetilde{\mathcal L}_N$ the same criterion with each maximum replaced by the exact conjugate on $K$.
Then $\E_N\widetilde{\mathcal L}_N=\mathcal L$ for each fixed candidate.
Normalize the potentials by $\psi(0,s)=0$, which changes neither their gradients nor these objectives.

For $\delta\in(0,1]$, let $M_n(\delta)$ be the smallest number of admissible representatives needed to approximate every normalized potential $(u,X)\mapsto\psi(u,f(X))/(2M)$ uniformly within $\delta$.
Thus this count measures the complexity of the potentials used by the complete posterior learner, including its encoder.
Assume the class conditions of Proposition~\ref{prop:d-affine-complexity}, including its normalized-potential stability condition.
For affine potentials, $p=d$, $s=d$, and $P_n$ counts the encoder and any auxiliary parameters; for directly learned raw potentials $\psi(u,X)$, put $s=nd_X+d$ and count any auxiliary parameters in $P_n$.
Use the corresponding regularity constant $A_n$ from that proposition, retaining the summary dimension and encoder sensitivity in the affine case.
Finally, assume $\Pr\{U\in B(u,r)\}\ge\kappa_dr^d$ for $u\in K$ and $0<r\le1$, where $B(u,r)$ is the Euclidean ball of radius $r$ centered at $u$.
The volume-uniform and radius-uniform sources permit $\kappa_d=2^{-d}$ and $\kappa_d=2^{-d}/d$, respectively.
Below, $C_d$ denotes a finite constant depending only on $d$ and $\kappa_d$.

\begin{claim}[Statistical and source-maximum errors]
\label{claim:c-dgqb-empirical}
Under these conditions,
\begin{equation}
\E_N\sup_{f,\psi}|\widetilde{\mathcal L}_N-\mathcal L|
\le3\left[
\sqrt{\frac{(P_n+1)\log(2A_nN)}N}
+\left\{\frac{A_n^s\log(2A_nN)}N\right\}^{1/(s+2)}
\right].
\label{eq:v3-qnn-product-rate}
\end{equation}
The actual training criterion satisfies
\[
\E_N\sup_{f,\psi}|\mathcal L_N-\mathcal L|
\le3\left[
\sqrt{\frac{(P_n+1)\log(2A_nN)}N}
+\left\{\frac{A_n^s\log(2A_nN)}N\right\}^{1/(s+2)}
\right]
+C_d\left\{\frac{\log(eN)}N\right\}^{1/d}.
\]
For the raw class, each supremum ranges over its admissible directly learned potentials.
\end{claim}

The first two terms control learning the encoder and potential, while the last controls the sample maximum.
Affinity changes the functional exponent from $1/(nd_X+d+2)$ to $1/(d+2)$ under the stated classes.
The $N^2$ source--observation combinations are formed from two independent samples of size $N$ and therefore do not provide $N^2$ independent observations.

\begin{proof}
Conditionally on the source bank, the unnormalized contribution of one simulator pair lies in $[-T,T+2B]$, by Fenchel's inequality and the $B$-Lipschitz bound on $\psi$.
Replacing one simulator pair therefore changes $\widetilde{\mathcal L}_N$ by at most $1/(2N)$.
Replacing one source draw changes only its product-average term, by at most $2B/(NV_n)\le1/(2N)$.
The bounded-differences exponential bound over these $2N$ independent inputs gives
\[
\E_N\exp\{\lambda(\widetilde{\mathcal L}_N-\mathcal L)\}
\le\exp\{\lambda^2/(16N)\},\qquad\lambda\in\R.
\]
Conjugation is 1-Lipschitz in potential values, so the normalized empirical and population objectives are each 1-Lipschitz in the covering metric.
Proposition~\ref{prop:d-affine-complexity} supplies the cover bound, and Claim~\ref{claim:b-cover-rate} gives \eqref{eq:v3-qnn-product-rate}.
In particular, its proof balances the errors at $\delta=\{A_n^s\log(2A_nN)/N\}^{1/(s+2)}$ when this is at most one; otherwise the bounded criterion gives the displayed bound directly.

For the source approximation, put $h_N=\sup_{u\in K}\min_{j\le N}\norm{u-U_j}$.
The conjugate integrand is $M$-Lipschitz, hence $0\le\widetilde{\mathcal L}_N-\mathcal L_N\le h_N/4$ uniformly over the candidate class.
An $r/2$-net of $K$ with at most $(5/r)^d$ points and the source small-ball condition give
\[
\Pr(h_N>r)\le(5/r)^d\exp\{-\kappa_dN(r/2)^d\},\qquad0<r\le1.
\]
Choose $r$ of order $\{\log(eN)/N\}^{1/d}$ with a sufficiently large constant depending only on $d$ and $\kappa_d$.
This tail bound and $h_N\le2$ yield $\E_Nh_N/4\le C_d\{\log(eN)/N\}^{1/d}$, proving the second bound.
For the stated small-ball constants, the ball of radius $r/2$ centered at $(1-r/2)u$ lies in $K\cap B(u,r)$, and the radius-uniform density is at least $1/d$ times the volume-uniform density.
\end{proof}

Under affinity, centered summaries satisfy $N^{-1}\sum_iS_i=0$ and hence $N^{-2}\sum_{i,j}b(U_j)^\top S_i=0$.
Thus the product-average term retains only $\varphi$, and the empirical criterion reduces to \eqref{eq:qr-loss}, the objective in Equation~(3.4) of \citet{kim2025quantile}, up to normalization.
Fitted centering and normalization parameters may depend on the training data if their possible values are included in the fixed covered class.

Under the conditions of Theorem~\ref{thm:v3-two-way}, the corresponding posterior and insufficiency risks satisfy
\[
\begin{aligned}
R_p\le\Gamma_n\Bigg(&\Delta_{\mathcal F}+\E_N\xi_N
+6\left[
\sqrt{\frac{(P_n+1)\log(2A_nN)}N}
+\left\{\frac{A_n^s\log(2A_nN)}N\right\}^{1/(s+2)}
\right]\\
&+2C_d\left\{\frac{\log(eN)}N\right\}^{1/d}\Bigg),
\qquad R_f\le2R_p.
\end{aligned}
\]
This bound retains the source-maximum error explicitly.
Unbounded observation or parameter domains, a fixed number of source points per maximum, and changing minibatch normalization require additional control of the corresponding discrepancies.

\subsection{Summary flow matching}
\label{app:v3-sfm-calibration}

In SFM, the parameter-displacement objective guides summary learning, and a block-triangular flow uses the learned summary for posterior sampling.
We examine these two steps and then control the empirical regression error.

\subsubsection{Two-way guidance}

We first show how the parameter regression loss penalizes posterior information discarded by the summary.
The comparison uses the best prediction of parameter displacement given the full observation as its benchmark.

Fix an encoder $f$ and let $S_1=f(X)$, where $(X,\theta_1)$ has the prior--simulator law.
The sources $S_0$ and $\theta_0$ are independent of each other and of $(X,\theta_1)$, and $t\sim\operatorname{Unif}[0,1]$ is independent of these variables.
Use the straight-line interpolants in \eqref{eq:sfm-interpolants}, with displacements $\Delta S=S_1-S_0$ and $\Delta\theta=\theta_1-\theta_0$, and assume finite second moments for $S_0,S_1,\theta_0,\theta_1$ and the fitted velocities.
Define the conditional displacement means and the associated benchmark and information loss by
\[
\begin{aligned}
m_X&=\E[\Delta\theta\mid t,\theta_t,X],
& m_S&=\E[\Delta\theta\mid t,\theta_t,S_1],\\
\mathcal L_{\theta,0}&=\E\norm{\Delta\theta-m_X}^2,
&\mathcal I_\theta(f)&=\E\norm{m_X-m_S}^2.
\end{aligned}
\]
The benchmark $\mathcal L_{\theta,0}$ is independent of $f$ because neither the parameter path nor its displacement uses the summary.
Write $\mathcal L_\theta(f,v^\theta)$ for the parameter loss in \eqref{eq:sfm-losses}.

\begin{proposition}[Two-way guidance in SFM]
\label{prop:c-sfm-information}
Under the preceding conditions,
\begin{equation}
\mathcal L_\theta(f,v^\theta)-\mathcal L_{\theta,0}
=\underbrace{\mathcal I_\theta(f)}_{\text{information-loss gap}}
+\underbrace{\E\norm{v^\theta(t,S_t,\theta_t)-m_S}^2}_{\text{moving-summary and fitting error}}.
\label{eq:sfm-parameter-information-gap}
\end{equation}
Both terms are nonnegative.
If $\theta_0\sim\mathcal N(0,I_d)$, then
\[
\mathcal I_\theta(f)=0
\quad\Longleftrightarrow\quad
\pi(\cdot\mid X)=\pi(\cdot\mid f(X))\quad\text{almost surely}.
\]
\end{proposition}

This is the population basis of two-way guidance in SFM.
The summary organizes velocity learning through $S_t$, while an insufficient summary incurs a positive parameter-regression penalty that changing the velocity alone cannot remove.
Note that the second term also reflects the source noise in $S_t$: it does not necessarily vanish even when $S_1$ is sufficient.
Thus, sufficiency removes the information-loss gap, while attaining the full-data benchmark additionally requires this second term to vanish.

\begin{proof}
Independence of $S_0$ gives the two orthogonality relations
\[
\E[\Delta\theta-m_X\mid t,\theta_t,X,S_0]=0,
\qquad
\E[m_X-m_S\mid t,\theta_t,S_1,S_0]=0.
\]
The predictor $v^\theta(t,S_t,\theta_t)$ is a function of either conditioning collection, since $S_t=(1-t)S_0+tS_1$ and $S_1=f(X)$.
Expanding its squared residual first around $m_X$ and then around $m_S$ eliminates both cross terms and proves \eqref{eq:sfm-parameter-information-gap}.

For the equivalence, suppose $\theta_0\sim\mathcal N(0,I_d)$.
At each $0<t<1$, the conditional density of $\theta_t$ given $X=x$ is
\[
p_t(z\mid x)=\int\varphi_{(1-t)^2I_d}(z-t\vartheta)\,\pi(\mathrm d\vartheta\mid x),
\]
where $\varphi_\Sigma$ is the centered Gaussian density with covariance $\Sigma$.
Define $p_t(z\mid S_1=s)$ by replacing the posterior in this integral with $\pi(\cdot\mid S_1=s)$.
Differentiating the Gaussian convolution and using $\Delta\theta=(\theta_1-\theta_t)/(1-t)$ yields
\begin{equation}
m_X(t,z,x)=\frac zt+\frac{1-t}{t}\nabla_z\log p_t(z\mid x),
\label{eq:sfm-gaussian-regression-score}
\end{equation}
with the same identity for $m_S$ and $p_t(z\mid S_1=s)$.
Gaussian smoothing justifies differentiation under the integral and gives strictly positive smooth densities.

If $\mathcal I_\theta(f)=0$, Fubini's theorem supplies a deterministic $t\in(0,1)$ for which $m_X(t,z,x)=m_S(t,z,f(x))$ for almost every $x$ and $p_t(\cdot\mid x)$-almost every $z$.
By positivity and continuity, the scores in \eqref{eq:sfm-gaussian-regression-score} agree everywhere in $z$ for almost every such $x$.
The two log densities therefore differ by a constant, which is zero because both densities integrate to one.
Taking characteristic functions and canceling the nonzero Gaussian factor $\exp\{-(1-t)^2\norm\omega^2/2\}$ gives equality of the posterior characteristic functions at every frequency $t\omega$.
Since $t>0$, this proves posterior sufficiency.
Conversely, posterior sufficiency gives the same conditional law of $(\theta_t,\Delta\theta)$ given $X$ and given $S_1$, hence $m_X=m_S$ almost surely.
\end{proof}

The training update makes this guidance explicit.
Suppose $f_\eta$ and $v^\theta_\beta$ are differentiable and differentiation can be interchanged with expectation.
Let $r_\theta=v^\theta_\beta(t,S_t,\theta_t)-\Delta\theta$, and write $D_\eta$ and $D_s$ for the corresponding Jacobian matrices.
The chain rule gives
\begin{equation}
\nabla_\eta\mathcal L_\theta
=2\E\left[t\{D_\eta f_\eta(X)\}^\top
\{D_s v^\theta_\beta(t,S_t,\theta_t)\}^\top r_\theta\right].
\label{eq:sfm-encoder-gradient}
\end{equation}
For an affine parameter velocity $v^\theta_\beta(t,s,\vartheta)=a_\beta(t,\vartheta)+B_\beta(t,\vartheta)s$, the Jacobian $D_s v^\theta_\beta$ is $B_\beta(t,\vartheta)$.
Thus the fitted coefficient matrix passes the parameter-displacement residual to the encoder, and the encoder supplies the summary used to learn that matrix.
The stop-gradient in \eqref{eq:sfm-losses} sends no encoder derivative from the summary loss, leaving summary learning to this parameter-directed update.

\subsubsection{Posterior accuracy}

We next show why the block-triangular joint transport produces posterior samples when conditioned on a learned summary.
This is the conditional-sampling property of \citet{baptista2024conditional}, applied to the dynamic construction of \citet{zhai2026cfm}.

Let $T_t$ be the flow of \eqref{eq:sfm-flow} and write
\[
T_t(s_0,u)=\{F_t(s_0),H_t(s_0,u)\},
\qquad
G(u,s)=H_1\{F_1^{-1}(s),u\}.
\]
Assume the joint velocity is continuous in $t$, globally Lipschitz in $(s,\vartheta)$ uniformly in $t$, and uniformly bounded at $(s,\vartheta)=(0,0)$.
These conditions give unique forward and backward flows; the summary component depends only on $s_0$, and $F_1$ is invertible.
Use independent sources as above and assume the joint endpoint is correct:
\begin{equation}
\Law\{T_1(S_0,\theta_0)\}=\Law(S_1,\theta_1).
\label{eq:sfm-exact-joint-transport}
\end{equation}
Draw fresh $U\sim\Law(\theta_0)$ independently of $X$.

\begin{proposition}[Posterior sampling from the joint flow]
\label{prop:c-sfm-conditional}
Under these conditions, for $\Law(S_1)$-almost every $s$,
\begin{equation}
\Law\{G(U,s)\}=\pi(\cdot\mid S_1=s).
\label{eq:sfm-exact-summary-posterior}
\end{equation}
Consequently,
\[
\E_XW_2^2\!\left\{\Law\{G(U,f(X))\mid X\},\pi(\cdot\mid X)\right\}
=\E_XW_2^2\!\left\{\pi(\cdot\mid f(X)),\pi(\cdot\mid X)\right\}.
\]
\end{proposition}

Correct joint transport leaves precisely the posterior information lost by the summary.
A posterior-sufficient summary therefore gives the full-data posterior through the backward-summary and forward-parameter sampling procedure in Algorithm~\ref{alg:sfm}.

\begin{proof}
Take $U$ independent of $S_0$ with the law of $\theta_0$.
By definition, $H_1(S_0,U)=G\{U,F_1(S_0)\}$, and $F_1(S_0)$ has the law of $S_1$ under \eqref{eq:sfm-exact-joint-transport}.
For bounded measurable functions $a$ and $b$, independence of the sources gives
\[
\begin{aligned}
\E\{a(S_1)b(\theta_1)\}
&=\E\left[a\{F_1(S_0)\}b\{G(U,F_1(S_0))\}\right]\\
&=\int a(s)\left[\int b\{G(u,s)\}\,F_U(\mathrm du)\right]
\Law(S_1)(\mathrm ds).
\end{aligned}
\]
Uniqueness of the conditional distribution proves \eqref{eq:sfm-exact-summary-posterior}, and substituting $s=f(X)$ gives the risk identity.
\end{proof}

The exact-joint-transport condition concerns the distribution generated by the flow.
The unrestricted population optimum of the summary regression is $\E[\Delta S\mid t,S_t]$, whereas the summary component of the unrestricted joint interpolation velocity is $\E[\Delta S\mid t,S_t,\theta_t]$.
Thus regression convergence alone does not establish \eqref{eq:sfm-exact-joint-transport}; posterior calibration for the fitted sampler remains a condition when applying Theorem~\ref{thm:v3-two-way}.

\subsubsection{Finite sample error}
\label{app:v3-sfm-rate}

We now control the finite-sample error of the two regression losses used in training.
The encoder affects both the summary path and the summary-displacement target, so both changes enter the argument.

Take $N\ge2$ independent and identically distributed records $(X,\theta_1,S_0,\theta_0,t)$ generated as above.
When the two losses use separate source and time draws for one simulated pair, include both sets of auxiliary draws in that pair's record.
Fix the candidate class before training and assume uniformly bounded source and target variables, summary outputs, and velocities.
Let $L_\theta$ and $L_S$ be common Lipschitz constants in the summary argument for the parameter and summary velocities.
Along with $r_\theta$ defined above, write $r_S=v^S(t,S_t)-\Delta S$, and let $M_\theta$ and $M_S$ bound their norms uniformly over records and candidates.
Choose $c_n>0$ with $c_n\ge M_\theta^2+M_S^2/q$ and define the normalized regression loss and its averages by
\[
\ell=\frac{\norm{r_\theta}^2+q^{-1}\norm{r_S}^2}{c_n},
\qquad
\overline{\mathcal L}=\E\ell,
\qquad
\overline{\mathcal L}_N=\frac1N\sum_{i=1}^N\ell_i,
\]
where $\ell_i$ is evaluated on the $i$th record using the candidate encoder and velocities.

For $\delta\in(0,1]$, let $M_n(\delta)$ be the smallest number of admissible representatives needed to approximate every normalized regression loss uniformly within $\delta$ over all records.
This count includes changes in both fitted velocities, their summary inputs, and the summary-displacement target.
Assume the class conditions of Proposition~\ref{prop:d-affine-complexity} for these complete losses.
For an affine parameter velocity, the coefficient functions depend on $(t,\theta_t)$, so $p=d+1$, $s=d+1$, and $P_n$ counts the encoder and summary-velocity parameters.
Use $A_n$ from that proposition, including the sensitivities of the moving summary and target.
For a raw joint-flow regression satisfying the corresponding complete-class conditions, use $s=nd_X+d+1$ and its own $P_n$ and $A_n$, retaining the cost of its observation-velocity network.

\begin{claim}[Statistical error of the SFM regression objective]
\label{claim:c-sfm-empirical}
Under these conditions,
\[
\E_N\sup_{f,v^S,v^\theta}
\left|\overline{\mathcal L}_N-\overline{\mathcal L}\right|
\le3\left[
\sqrt{\frac{(P_n+1)\log(2A_nN)}N}
+\left\{\frac{A_n^s\log(2A_nN)}N\right\}^{1/(s+2)}
\right].
\]
The same bound holds for the corresponding raw regression class with its own supremum and constants.
\end{claim}

This controls the statistical error of the actual regression objective, including its learned inputs and targets.
Affinity changes the functional exponent from $1/(nd_X+d+3)$ to $1/(d+3)$, while retaining the encoder and summary-velocity costs.

\begin{proof}
Consider two candidates whose uniform encoder and velocity distances are at most $\delta_f$, $\delta_\theta$, and $\delta_S$, using Euclidean output norms.
On each record, their summary paths and summary-displacement targets differ by at most $\delta_f$, while their parameter paths and parameter-displacement targets coincide.
Therefore,
\[
\norm{r_\theta-r'_\theta}\le\delta_\theta+L_\theta\delta_f,
\qquad
\norm{r_S-r'_S}\le\delta_S+(L_S+1)\delta_f,
\]
and the uniform residual bounds give
\[
\norm{\ell-\ell'}_\infty
\le\frac{2M_\theta}{c_n}(\delta_\theta+L_\theta\delta_f)
+\frac{2M_S}{qc_n}\{\delta_S+(L_S+1)\delta_f\}.
\]
These inequalities identify the moving-input and target terms included in the complete-loss stability condition of Proposition~\ref{prop:d-affine-complexity}.
Since $\ell\in[0,1]$, Hoeffding's lemma gives $\E_N\exp\{\lambda(\overline{\mathcal L}_N-\overline{\mathcal L})\}\le\exp\{\lambda^2/(8N)\}$ for each candidate and every $\lambda\in\R$.
Proposition~\ref{prop:d-affine-complexity} and Claim~\ref{claim:b-cover-rate} give the result, with the balancing resolution $\delta=\{A_n^s\log(2A_nN)/N\}^{1/(s+2)}$ when it is at most one.
The bounded loss handles the remaining case.
For the raw class, the observation path and displacement are fixed by each record, and the same argument applies to its complete velocity loss under the stated conditions.
\end{proof}

This empirical result concerns the sum of both regression losses; Proposition~\ref{prop:c-sfm-information} concerns the parameter loss and its full-data benchmark.
Applying Theorem~\ref{thm:v3-two-way} additionally requires posterior calibration and empirical optimization accuracy for the chosen complete criterion.
The detached updates must satisfy that optimization requirement separately, since the summary loss contributes no encoder gradient.
With these conditions, zero approximation gap, and the optimization and objective-approximation allowances in Corollary~\ref{cor:v3-affine-rate}, the affine-summary posterior bound is
\[
R_p\le\Gamma_n\!\left(10\left[
\sqrt{\frac{(P_n+1)\log(2A_nN)}N}
+\left\{\frac{A_n^{d+1}\log(2A_nN)}N\right\}^{1/(d+3)}
\right]\right),\qquad R_f\le2R_p.
\]
Here $\Gamma_n(t)=\Gamma_n^{\rm original}(c_nt)$ retains the regression-loss normalization in the calibration.
For the raw class, replace $d+1$ by $nd_X+d+1$ and use its own class constants and calibration.
Unbounded inputs, including Gaussian sources, require separate tail or truncation control for the empirical rate.

\subsection{Summary Bayesian GAN}
\label{app:v3-bgan-calibration}

SB-GAN learns an affine-summary generator through the Wasserstein minimax objective of \citet[Section~2.2]{wang2026generative}.
We show how the full-observation critic guides summary learning, then connect the adversarial objective to posterior accuracy and its empirical approximation.

\subsubsection{Two-way guidance}

We first show why the critic can detect posterior information discarded by the generator's summary.
The key is that the critic receives the full observation $X$, even though the generator depends on $X$ only through $S=f(X)$.

Fix $f$ and $G$, and assume $\E\norm\theta+\E\norm{G(U,S)}<\infty$.
Let $D$ range over all jointly measurable critics that are 1-Lipschitz in $\theta$ for every fixed $X$, with $D(X,0)=0$.
This normalization leaves the objective unchanged because the simulated and generated pairs have the same observation marginal.
Define the ideal objective and its full-data benchmark by
\[
\mathcal L(f,G)=\sup_D\E\big[D(X,\theta)-D\{X,G(U,S)\}\big],
\qquad \mathcal L_0=0.
\]

\begin{proposition}[Two-way guidance in SB-GAN]
\label{prop:c-bgan-information}
Under these conditions,
\begin{equation}
\mathcal L(f,G)=\E_XW_1\{\pi(\cdot\mid X),\pi_{f,G}(\cdot\mid X)\},
\label{eq:app-bgan-conditional-objective}
\end{equation}
and
\begin{equation}
\frac12\E_XW_1\{\pi(\cdot\mid X),\pi(\cdot\mid f(X))\}
\le\mathcal L(f,G).
\label{eq:app-bgan-information-penalty}
\end{equation}
\end{proposition}

This is the population basis of two-way guidance in SB-GAN.
Affinity specifies how the learned summary enters posterior generation, while the full-observation critic penalizes posterior differences that this summary fails to retain.
An insufficient summary incurs a strictly positive loss even after optimizing the generator.
If the summary is posterior sufficient and its posterior is representable by the generator class, zero loss is attainable.

\begin{proof}
Conditional Kantorovich--Rubinstein duality bounds every admissible critic's objective by the right side of \eqref{eq:app-bgan-conditional-objective}.
For the reverse inequality, take a countable family of normalized 1-Lipschitz functions on $\R^d$ that is dense for uniform convergence on compact sets and contains zero.
The bound $|D(\theta)|\le\norm\theta$ and finite first moments ensure that this family determines $W_1$.
At each $X$, select the smallest index maximizing the conditional integral difference among the first $m$ functions.
This defines a jointly measurable critic whose objective increases to the conditional Wasserstein expectation as $m\to\infty$.
Monotone convergence proves \eqref{eq:app-bgan-conditional-objective}.

The summary posterior is the mixture of the full posteriors conditional on $S$.
Since the generated law is fixed given $S$, mixing conditional optimal couplings gives
\[
W_1\{\pi(\cdot\mid S),\pi_{f,G}(\cdot\mid S)\}
\le\E\big[W_1\{\pi(\cdot\mid X),\pi_{f,G}(\cdot\mid S)\}\mid S\big].
\]
Such measurable couplings exist on Euclidean parameter spaces under the stated moments \citep[Corollary~5.22]{villani2009optimal}.
Averaging this inequality and applying the triangle inequality between the full posterior, generated law, and summary posterior yields \eqref{eq:app-bgan-information-penalty}.
\end{proof}

The generator update passes this critic signal to the summary network.
For the affine generator in \eqref{eq:sbgan-generator}, the chain rule at differentiability points gives
\[
\nabla_s[-D\{X,a_\beta(U)+B_\beta(U)s\}]
=-B_\beta(U)^\top\nabla_\theta D\{X,a_\beta(U)+B_\beta(U)s\}.
\]
Backpropagating this derivative through $s=f_\eta(X)$ updates the encoder according to discrepancies detected by the critic.

\subsubsection{Posterior accuracy}

The adversarial objective controls conditional $W_1$ error.
We now give conditions that convert this control into the $W_2$ posterior guarantee used in Theorem~\ref{thm:v3-two-way}.

First suppose that all true and generated parameter laws are supported on a common set of diameter $D_\Theta<\infty$.

\begin{claim}[Posterior error on bounded parameter spaces]
\label{claim:c-bgan-bounded}
Under this condition,
\[
\E_XW_2^2\{\pi_{f,G}(\cdot\mid X),\pi(\cdot\mid X)\}
\le D_\Theta\mathcal L(f,G).
\]
\end{claim}

Thus the ideal adversarial loss gives a linear bound on posterior error, providing the calibration required by Theorem~\ref{thm:v3-two-way}.

\begin{proof}
For each $X$, take a measurable conditional $W_1$-optimal coupling.
The inequality $\norm{y-y'}^2\le D_\Theta\norm{y-y'}$ on the common support bounds its squared transport cost by $D_\Theta$ times its $W_1$ cost.
The optimal $W_2^2$ cost is no larger; averaging and using \eqref{eq:app-bgan-conditional-objective} proves the claim.
\end{proof}

For unbounded parameter laws, suppose instead that some $r>2$ and $M_r<\infty$ satisfy
\[
\E\norm\theta^r+\sup_{f,G}\E\norm{G(U,f(X))}^r\le M_r.
\]

\begin{claim}[Posterior error under a moment bound]
\label{claim:c-bgan-moments}
Under this condition,
\[
\E_XW_2^2\{\pi_{f,G}(\cdot\mid X),\pi(\cdot\mid X)\}
\le(2^{r-1}M_r)^{1/(r-1)}\{\mathcal L(f,G)\}^{(r-2)/(r-1)}.
\]
\end{claim}

A uniform moment bound of order greater than two therefore gives posterior calibration on an unbounded parameter space.
Both displayed bounds are continuous, nondecreasing, and concave functions of the loss, and vanish at zero.

\begin{proof}
Average measurable conditional $W_1$-optimal couplings over $X$, and let $Y$ be the distance between their paired parameter draws.
Under this coupling, $\E Y=\mathcal L(f,G)$ and $\E Y^r\le2^{r-1}M_r$.
H\"older's inequality gives
\[
\E Y^2\le(\E Y)^{(r-2)/(r-1)}(\E Y^r)^{1/(r-1)}
\le(2^{r-1}M_r)^{1/(r-1)}\{\mathcal L(f,G)\}^{(r-2)/(r-1)}.
\]
The same coupling bounds the average conditional $W_2^2$ error by $\E Y^2$.
\end{proof}

\subsubsection{Finite sample error}
\label{app:v3-bgan-rate}

We next control the finite-sample error when the critic is a neural network.
This requires a cover of the complete adversarial loss, including the encoder, generator, and critic.

Take $N\ge2$ independent and identically distributed records $(X_i,\theta_i,U_i)$ formed from prior--simulator pairs and independent source draws on a bounded source domain.
Assume $\norm\theta\le T$ and $\norm{G(U,f(X))}\le T$ uniformly over the generator class, for $T>0$.
In this subsection, $D$ ranges over a fixed neural critic class that is exactly 1-Lipschitz in $\theta$ and satisfies $D(X,0)=0$.
Fix all candidate classes before drawing the training records and define
\[
\begin{aligned}
\ell_{f,G,D}(X,\theta,U)
&=\frac{D(X,\theta)-D\{X,G(U,f(X))\}}{4T},\\
\overline{\mathcal L}(f,G)&=\sup_D\E\ell_{f,G,D},
\qquad
\overline{\mathcal L}_N(f,G)=\sup_D\frac1N\sum_{i=1}^N\ell_{f,G,D}(X_i,\theta_i,U_i).
\end{aligned}
\]
For $\delta\in(0,1]$, let $M_n(\delta)$ be the smallest number of admissible representatives needed to approximate every normalized loss uniformly within $\delta$ over all records.
This count includes the encoder, generator, and critic.
Assume the class conditions of Proposition~\ref{prop:d-affine-complexity} for these losses.
For an affine generator, $p$ is the source dimension, $s=p$, and $P_n$ counts the encoder and critic parameters.
For directly learned raw generators $G(U,X)$, use $s=nd_X+p$ and count the critic parameters in $P_n$.
Use the respective $A_n$ from that proposition, including critic sensitivity and, for the affine class, the summary dimension and encoder sensitivity.

\begin{claim}[Statistical error of the adversarial objective]
\label{claim:c-bgan-empirical}
Under these conditions,
\[
\E_N\sup_{f,G}\left|\overline{\mathcal L}_N(f,G)-\overline{\mathcal L}(f,G)\right|
\le3\left[
\sqrt{\frac{(P_n+1)\log(2A_nN)}N}
+\left\{\frac{A_n^s\log(2A_nN)}N\right\}^{1/(s+2)}
\right].
\]
For the raw class, the supremum ranges over its admissible directly learned generators.
\end{claim}

Maximizing over the critic preserves the statistical bound when the critic is included in the loss cover.
Affinity reduces the generator's functional dimension from $nd_X+p$ to $p$, while the critic's learning cost remains because it receives the full observation needed to detect posterior information loss.

\begin{proof}
For two candidates, the Lipschitz property gives
\[
\norm{\ell_{f,G,D}-\ell_{f',G',D'}}_\infty
\le\frac{2\norm{D-D'}_\infty
+\norm{G(\cdot,f(\cdot))-G'(\cdot,f'(\cdot))}_\infty}{4T}.
\]
Thus the class conditions in Proposition~\ref{prop:d-affine-complexity} transfer the generator and critic covers to the complete loss.
Write $P_N$ and $P$ for empirical and population averages over the independent records.
Taking the critic supremum satisfies
\[
\sup_{f,G}\left|\sup_D P_N\ell_{f,G,D}-\sup_D P\ell_{f,G,D}\right|
\le\sup_{f,G,D}|(P_N-P)\ell_{f,G,D}|.
\]
Since $\ell_{f,G,D}\in[-1/2,1/2]$, Hoeffding's lemma gives the centered exponential bound $\exp\{\lambda^2/(8N)\}$ for each empirical loss average.
Proposition~\ref{prop:d-affine-complexity} and Claim~\ref{claim:b-cover-rate} then prove the claim.
The balancing resolution is $\delta=\{A_n^s\log(2A_nN)/N\}^{1/(s+2)}$ when this is at most one; otherwise the bounded objective gives the result directly.
\end{proof}

To obtain posterior guarantees, the neural-critic approximation error and errors from the sampled gradient penalty, inner optimization, and numerical evaluation must satisfy the aggregate allowance \eqref{eq:v3-auxiliary-criterion-error}, using the same normalization $4T$.
The optimization gap $\xi_N$ in Theorem~\ref{thm:v3-two-way} is measured against the empirical critic supremum; the loss against the last trained critic alone does not establish that gap.
Suppose $\Delta_{\mathcal F}=0$ for the normalized ideal objective, and each of that aggregate allowance and $\E_N\xi_N$ is no larger than the sum in square brackets in the preceding claim.
If $D_\Theta$ is the diameter of the common parameter support, which is at most $2T$, Claim~\ref{claim:c-bgan-bounded} and Corollary~\ref{cor:v3-affine-rate} give
\[
R_p\le40TD_\Theta\left[
\sqrt{\frac{(P_n+1)\log(2A_nN)}N}
+\left\{\frac{A_n^s\log(2A_nN)}N\right\}^{1/(s+2)}
\right],\qquad R_f\le2R_p.
\]
Thus the complete learning conditions give an explicit bound on both posterior error and summary insufficiency.
The moment-based calibration in Claim~\ref{claim:c-bgan-moments} applies to unbounded parameter laws; unbounded source, observation, or parameter domains require additional tail control for the empirical rate.

\subsection{Neural posterior estimation}
\label{app:npe-population}

\subsubsection{Two-way guidance}

We show how the conditional-density objective of NPE penalizes posterior information lost by the summary.
We use the prior-simulation objective in Section~2.2 of \citet{greenberg2019automatic}, with both the summary $S=f_\eta(X)$ and the conditional density $\pi_\beta(\theta\mid S)$ learned jointly.
Assume that the full-data posterior, summary posterior, and fitted conditional law have measurable densities with respect to a common $\sigma$-finite measure on the parameter space.
Suppose their log densities are absolutely integrable under the simulated joint law of $(\theta,X)$, for every candidate under consideration.
These conditions ensure that the loss differences below are finite.
With natural logarithms, the population objective and its full-data optimum are
\[
\mathcal L_{\mathrm{NPE}}(\eta,\beta)
=-\E_{\theta,X}\log\pi_\beta(\theta\mid f_\eta(X)),
\qquad
\mathcal L_{\mathrm{NPE}}^*=-\E_{\theta,X}\log\pi(\theta\mid X).
\]
Write $\operatorname{KL}(P\|Q)$ for Kullback--Leibler divergence.

\begin{proposition}[Information loss and posterior fitting]
\label{prop:c-npe-information}
Under these conditions,
\begin{equation}
\begin{aligned}
\mathcal L_{\mathrm{NPE}}(\eta,\beta)-\mathcal L_{\mathrm{NPE}}^*
={}&\E_X\operatorname{KL}\{\pi(\cdot\mid X)\,\|\,\pi_\beta(\cdot\mid S)\}\\
={}&\E_X\operatorname{KL}\{\pi(\cdot\mid X)\,\|\,\pi(\cdot\mid S)\}
+\E_X\operatorname{KL}\{\pi(\cdot\mid S)\,\|\,\pi_\beta(\cdot\mid S)\}.
\end{aligned}
\label{eq:app-npe-decomposition}
\end{equation}
Both terms in the decomposition are nonnegative.
The first is zero if and only if $\pi(\cdot\mid X)=\pi(\cdot\mid S)$ for prior-predictive almost every $X$.
The second is zero if and only if $\pi_\beta(\cdot\mid S)=\pi(\cdot\mid S)$ almost surely.
\end{proposition}

This decomposition explains two-way guidance in NPE.
The summary supplies the conditioning information for posterior density estimation, while the negative log-likelihood penalizes posterior information that the summary discards.
Note that improving the density estimator with the summary fixed cannot remove the first gap.
Thus, approaching the full-data optimum requires both a summary with little information loss and an accurate conditional density given that summary.

\begin{proof}
Conditioning the log-density ratio on $X$ gives the first equality in \eqref{eq:app-npe-decomposition}.
Insert the summary posterior into that ratio to obtain
\[
\E_{\theta,X}\log\frac{\pi(\theta\mid X)}{\pi_\beta(\theta\mid S)}
=\E_{\theta,X}\log\frac{\pi(\theta\mid X)}{\pi(\theta\mid S)}
+\E_{\theta,X}\log\frac{\pi(\theta\mid S)}{\pi_\beta(\theta\mid S)}.
\]
Condition the first term on $X$ and the second on $S$.
This yields the two divergences, since the conditional parameter law given $S$ is $\pi(\cdot\mid S)$.
Both outer expectations are written as $\E_X$ because $S=f_\eta(X)$.
Nonnegativity and the equality conditions for KL divergence prove the remaining statements.
\end{proof}

The same objective supplies the encoder's training signal.
At differentiability points, each simulated pair contributes
\[
\nabla_\eta[-\log\pi_\beta(\theta\mid f_\eta(X))]
=-\{D_\eta f_\eta(X)\}^{\top}
\left.\nabla_s\log\pi_\beta(\theta\mid s)\right|_{s=f_\eta(X)},
\]
where $D_\eta f_\eta(X)$ is the encoder Jacobian.
This derivative updates the summary according to how well the conditional density explains the simulated parameter.
The population decomposition gives this training signal its information-preserving justification; no affine dependence on $S$ is required.

\subsubsection{Posterior accuracy}

We next show when a small negative log-likelihood excess implies a small $W_2$ posterior error.
Assume that the posterior sampler generates the fitted conditional density exactly, as a conditional normalizing flow does by transforming its base draws.
First suppose the true and fitted conditional laws are supported on a common set of diameter $D_\Theta<\infty$, uniformly over the candidate class.

\begin{claim}[Posterior error on bounded parameter spaces]
\label{claim:c-npe-bounded}
Under this condition,
\[
\E_XW_2^2\{\pi_\beta(\cdot\mid S),\pi(\cdot\mid X)\}
\le D_\Theta^2
\left\{\frac{\mathcal L_{\mathrm{NPE}}(\eta,\beta)-\mathcal L_{\mathrm{NPE}}^*}{2}\right\}^{1/2}.
\]
\end{claim}

Thus, the density-estimation objective controls the full posterior error, including the error caused by an insufficient summary.

\begin{proof}
Write $\operatorname{TV}(P,Q)=\sup_A|P(A)-Q(A)|$ for total variation distance.
A maximal coupling and Pinsker's inequality give
\[
W_2^2(P,Q)\le D_\Theta^2\operatorname{TV}(P,Q)
\le D_\Theta^2\sqrt{\operatorname{KL}(P\|Q)/2}.
\]
Apply this inequality to $P=\pi(\cdot\mid X)$ and $Q=\pi_\beta(\cdot\mid S)$, then average over $X$.
Jensen's inequality and \eqref{eq:app-npe-decomposition} give the claim.
\end{proof}

For unbounded parameter laws, suppose instead that some $r>2$ and $M_r<\infty$ satisfy
\[
\E\norm\theta^r
+\sup_{\eta,\beta}\E_X\int\norm\vartheta^r\pi_\beta(d\vartheta\mid f_\eta(X))
\le M_r.
\]
This bounds the prior moment and the average fitted-posterior moment uniformly over the candidate class.

\begin{claim}[Posterior error under a moment bound]
\label{claim:c-npe-moments}
Under this condition,
\[
\E_XW_2^2\{\pi_\beta(\cdot\mid S),\pi(\cdot\mid X)\}
\le(2^{r-1}M_r)^{2/r}
\left\{\frac{\mathcal L_{\mathrm{NPE}}(\eta,\beta)-\mathcal L_{\mathrm{NPE}}^*}{2}\right\}^{(r-2)/(2r)}.
\]
\end{claim}

This result allows unbounded posteriors, including Gaussian laws with the stated moment control.
Both posterior bounds are continuous, nondecreasing, and concave functions of the loss excess, and vanish at zero.
They therefore provide the calibration required by Theorem~\ref{thm:v3-two-way}.

\begin{proof}
For each $X$, couple the true and fitted posterior maximally, using their common density minimum for the shared part.
The measurable conditional densities make this a measurable coupling, which we average over $X$.
Let $Y$ be the distance between its paired parameter draws.
Pinsker's and Jensen's inequalities, followed by \eqref{eq:app-npe-decomposition}, yield
\[
\Pr(Y>0)=\E_X\operatorname{TV}\{\pi(\cdot\mid X),\pi_\beta(\cdot\mid S)\}
\le\left\{\frac{\mathcal L_{\mathrm{NPE}}(\eta,\beta)-\mathcal L_{\mathrm{NPE}}^*}{2}\right\}^{1/2}.
\]
The moment condition gives $\E Y^r\le2^{r-1}M_r$.
H\"older's inequality on the event $\{Y>0\}$ gives
\[
\E Y^2\le(\E Y^r)^{2/r}\Pr(Y>0)^{1-2/r}.
\]
The average conditional $W_2^2$ error is no larger than this coupling's squared cost $\E Y^2$, proving the claim.
\end{proof}

\subsubsection{Finite sample error}

We now bound estimation of the NPE objective in terms of the total number of fitted parameters.
Fix the candidate class before drawing $N\ge2$ independent prior--simulator pairs $(\theta_i,X_i)$, and define
\[
\mathcal L_{\mathrm{NPE},N}(\eta,\beta)
=-\frac1N\sum_{i=1}^N\log\pi_\beta(\theta_i\mid f_\eta(X_i)).
\]
Assume that the negative log densities of all candidates lie in a common interval of length $V_n>0$ over all possible training pairs.
Write $\zeta$ for the complete parameter vector of the encoder and conditional density network, including any fitted normalization parameters, and let $W_n$ be its dimension.
Thus $W_n$ counts all fitted parameters, whereas $P_n$ in Corollary~\ref{cor:v3-affine-rate} counts only the encoder and auxiliary-network parameters.
Assume that the candidate parameter set is contained in $[-L_n,L_n]^{W_n}$ for some $L_n\ge1$ and that the complete normalized losses
\[
\ell_\zeta(\theta,X)
=-V_n^{-1}\log\pi_\beta(\theta\mid f_\eta(X))
\]
satisfy
\[
\sup_{\theta,X}|\ell_\zeta(\theta,X)-\ell_{\zeta'}(\theta,X)|
\le L_n\|\zeta-\zeta'\|_\infty
\]
for all candidates $\zeta,\zeta'$.
For this class, set $A_n=1+2L_n^2$.
These conditions bound the sensitivity of the entire loss to every fitted parameter and impose no affine dependence on the summary.

\begin{claim}[Statistical error of the density-estimation objective]
\label{claim:c-npe-empirical}
Under these conditions,
\[
\E_N\sup_{\eta,\beta}
\left|\mathcal L_{\mathrm{NPE},N}(\eta,\beta)-\mathcal L_{\mathrm{NPE}}(\eta,\beta)\right|
\le V_n\left[\frac2N+\sqrt{\frac{\log2+W_n\log(A_nN)}{2N}}\right].
\]
\end{claim}

The statistical term counts the encoder and density network together, including the cost of learning the summary.
For fixed $n$ and a fixed class, it decreases at order $\sqrt{\log N/N}$.

\begin{proof}
The normalized losses take values in an interval of length one.
For each fixed candidate, Hoeffding's lemma bounds the moment generating function of its centered empirical mean by $\exp\{\lambda^2/(8N)\}$.
Claim~\ref{claim:b-parametric-cover} bounds the logarithm of the covering number by $W_n\log(A_n/\delta)$ for $0<\delta\le1$.
Apply Claim~\ref{claim:b-uniform-deviation} with concentration constant $c=1/2$ and choose $\delta=1/N$.
The two covering errors sum to $2/N$, and the deviation over the representatives is at most $\sqrt{\{\log2+W_n\log(A_nN)\}/(2N)}$.
Multiplying by $V_n$ proves the claim.
\end{proof}

Take $\Delta_{\mathcal F}$ and $\xi_N$ in Theorem~\ref{thm:v3-two-way} for the NPE population and empirical objectives above, with $\mathcal L_0=\mathcal L_{\mathrm{NPE}}^*$.
Under the bounded-support condition of Claim~\ref{claim:c-npe-bounded} and the other hypotheses of Theorem~\ref{thm:v3-two-way},
\[
R_p\le D_\Theta^2
\left[\frac{\Delta_{\mathcal F}+\E_N\xi_N}{2}
+V_n\left\{\frac2N+\sqrt{\frac{\log2+W_n\log(A_nN)}{2N}}\right\}\right]^{1/2},
\qquad R_f\le2R_p.
\]
This follows by applying the posterior calibration to the expected loss excess and using concavity.
Under Claim~\ref{claim:c-npe-moments}, the same square bracket is raised to $(r-2)/(2r)$ and multiplied by $(2^{r-1}M_r)^{2/r}$ instead.
For fixed $n$ and classes with finite $W_n$, $L_n$, and $V_n$, the statistical term tends to zero as $N\to\infty$.
Consequently, $\Delta_{\mathcal F}=0$ and $\E_N\xi_N\to0$ imply $R_p\to0$ and $R_f\to0$ under the respective calibration conditions.

The bounded-loss and uniform-sensitivity assumptions are sufficient conditions for this rate and must be verified for the complete density architecture.
A finite number of network parameters alone does not establish them.
Normalizing flows with Gaussian tails generally have unbounded negative log densities, which require additional tail control for concentration beyond the posterior moment condition above.

\subsection{Neural ratio estimation}
\label{app:nre-population}

\subsubsection{Two-way guidance}

We show how distinguishing dependent parameter--observation pairs from independent pairs guides the summary toward posterior sufficiency.
This is the classification problem used by \citet[Section~3.1 and Appendix~B]{hermans2020likelihoodfreemcmcamortizedapproximate} to estimate the likelihood-to-evidence ratio.
Hold the networks fixed and write $S=f_\eta(X)$.
Let $P$ be the simulated joint law of $(\theta,X)$, let $P_X$ be its observation marginal, and set $Q=\pi\otimes P_X$ and $M=(P+Q)/2$.
Thus $P$ generates dependent pairs, $Q$ generates independent pairs, and $M$ averages these two distributions with equal class probabilities.
The notation $\E_P$, $\E_Q$, and $\E_M$ denotes expectation under these respective laws; $\E_X$ retains its prior-predictive meaning.

Assume $P\ll Q$ on standard Borel spaces.
Write $r_X=dP/dQ$ and let $r_S(\theta,S)$ be the posterior-to-prior density ratio given the summary, so
\[
\pi(d\theta\mid X)=r_X(\theta,X)\pi(d\theta),
\qquad
\pi(d\theta\mid S)=r_S(\theta,S)\pi(d\theta).
\]
The Bayes probabilities of the dependent class are
\[
D_X^*(\theta,X)=\frac{r_X(\theta,X)}{1+r_X(\theta,X)},
\qquad
D_S^*(\theta,S)=\frac{r_S(\theta,S)}{1+r_S(\theta,S)}.
\]
For $a,b\in[0,1]$, define the Bernoulli divergence $k(a,b)=a\log(a/b)+(1-a)\log\{(1-a)/(1-b)\}$ with the usual endpoint conventions, and let $h(a)=-a\log a-(1-a)\log(1-a)$ be binary entropy.
The population objective and its full-data optimum are
\[
\mathcal L_{\mathrm{NRE}}(\eta,\beta)
=-\tfrac12\E_P\log D_\beta(\theta,S)
-\tfrac12\E_Q\log\{1-D_\beta(\theta,S)\},
\qquad
\mathcal L_{\mathrm{NRE}}^*=\E_Mh(D_X^*).
\]
Assume $0<D_\beta<1$ and finite population binary cross-entropy for every candidate under consideration.

\begin{proposition}[Information loss and classifier fitting]
\label{prop:c-nre-information}
Under these conditions,
\begin{equation}
\begin{aligned}
\mathcal L_{\mathrm{NRE}}(\eta,\beta)-\mathcal L_{\mathrm{NRE}}^*
={}&\E_Mk\{D_X^*(\theta,X),D_\beta(\theta,S)\}\\
={}&\E_Mk\{D_X^*(\theta,X),D_S^*(\theta,S)\}
+\E_Mk\{D_S^*(\theta,S),D_\beta(\theta,S)\}.
\end{aligned}
\label{eq:app-nre-decomposition}
\end{equation}
Both terms in the decomposition are nonnegative.
The first is zero if and only if $\pi(\cdot\mid X)=\pi(\cdot\mid S)$ for prior-predictive almost every $X$.
The second is zero if and only if $D_\beta(\theta,S)=D_S^*(\theta,S)$ $M$-almost surely.
\end{proposition}

This is two-way guidance in NRE.
The summary supplies the information used to distinguish dependent and independent pairs, while binary cross-entropy penalizes posterior information lost through the summary.
An accurate classifier for a fixed insufficient summary cannot remove the first gap.
In particular, collapsing the summary to a constant incurs a positive information-loss gap whenever the full posterior depends on $X$.

\begin{proof}
Conditioning under the independent-pair law gives $r_S(\theta,S)=\E_Q[r_X(\theta,X)\mid\theta,S]$.
Since $dM/dQ=(1+r_X)/2$, it follows that
\[
\E_M[D_X^*(\theta,X)\mid\theta,S]
=\frac{\E_Q[r_X(\theta,X)\mid\theta,S]}{1+\E_Q[r_X(\theta,X)\mid\theta,S]}
=D_S^*(\theta,S).
\]
The cross-entropy equals $\E_M[-D_X^*\log D_\beta-(1-D_X^*)\log(1-D_\beta)]$.
Subtracting $\E_Mh(D_X^*)$ proves the first equality in \eqref{eq:app-nre-decomposition}.
Conditioning on $(\theta,S)$ replaces $D_X^*$ in the cross-entropy by $D_S^*$.
Subtracting and adding $\E_Mh(D_S^*)$ then yields the classifier-fitting term and the entropy difference
\[
\E_Mh(D_S^*)-\E_Mh(D_X^*)
=\E_Mk(D_X^*,D_S^*).
\]
This identity also holds at the endpoints: conditional on $D_S^*=0$ or $1$, the same conditional-expectation relation forces $D_X^*=D_S^*$ almost surely.

The equality condition for Bernoulli divergence makes the first gap zero exactly when $D_X^*=D_S^*$ $M$-almost surely.
The measures $M$ and $Q$ are equivalent, and $r\mapsto r/(1+r)$ is injective for finite $r\ge0$.
Hence this equality is equivalent to $r_X(\theta,X)=r_S(\theta,f_\eta(X))$ $Q$-almost surely, which is posterior sufficiency by the conditional-density identities and Fubini's theorem.
The equality condition for the second divergence gives the remaining statement.
\end{proof}

Both types of training pairs update the encoder.
For class label $y\in\{0,1\}$, the derivative with respect to the summary at differentiability points is
\[
\nabla_s[-y\log D_\beta(\theta,s)-(1-y)\log\{1-D_\beta(\theta,s)\}]
=\{D_\beta(\theta,s)-y\}
\nabla_s\log\frac{D_\beta(\theta,s)}{1-D_\beta(\theta,s)}.
\]
Backpropagating through $s=f_\eta(X)$ therefore teaches the summary to retain the parameter--observation dependence detected by the classifier.
The population decomposition identifies exactly when that retained information is posterior sufficient.

\subsubsection{Posterior accuracy}

We next turn classification error into posterior error after normalizing the fitted likelihood-to-evidence ratio.
Define the fitted odds, their prior normalizer, and the resulting posterior law by
\[
\widehat r_\beta(\theta,s)=\frac{D_\beta(\theta,s)}{1-D_\beta(\theta,s)},
\qquad
Z_\beta(s)=\int\widehat r_\beta(\theta,s)\pi(d\theta),
\qquad
\widetilde\pi_\beta(d\theta\mid s)
=\frac{\widehat r_\beta(\theta,s)}{Z_\beta(s)}\pi(d\theta).
\]
At the summary Bayes classifier, these odds equal $r_S$, the normalizer equals one, and this law equals $\pi(\cdot\mid S=s)$.
For the following bounds, assume $0<D_\beta(\theta,f_\eta(X))\le1-\tau_n$ uniformly over the candidate class, for some $0<\tau_n<1/2$.
This condition keeps the fitted odds bounded and ensures $0<Z_\beta(S)<\infty$ almost surely.
First suppose the prior is supported on a set of diameter $D_\Theta<\infty$.

\begin{claim}[Posterior error on bounded parameter spaces]
\label{claim:c-nre-bounded}
Under these conditions,
\[
\E_XW_2^2\{\widetilde\pi_\beta(\cdot\mid S),\pi(\cdot\mid X)\}
\le\frac{D_\Theta^2}{\tau_n}
\sqrt{2\{\mathcal L_{\mathrm{NRE}}(\eta,\beta)-\mathcal L_{\mathrm{NRE}}^*\}}.
\]
\end{claim}

The classification objective therefore controls posterior error when conversion from classifier probabilities to odds is stable.
The bound does not impose an upper bound on the true Bayes classifier; its constant controls the fitted classifier.

\begin{proof}
Using $dQ=2(1-D_X^*)dM$, we obtain
\[
\E_Q|r_X-\widehat r_\beta|
=2\E_M\frac{|D_X^*-D_\beta|}{1-D_\beta}
\le\frac{\sqrt{2\{\mathcal L_{\mathrm{NRE}}(\eta,\beta)-\mathcal L_{\mathrm{NRE}}^*\}}}{\tau_n}.
\]
The last step uses $k(a,b)\ge2(a-b)^2$, Cauchy--Schwarz, and \eqref{eq:app-nre-decomposition}.
For each $X$, $\int r_X\,d\pi=1$ and
\[
\begin{aligned}
\int\left|\frac{\widehat r_\beta}{Z_\beta(S)}-r_X\right|d\pi
&\le |1-Z_\beta(S)|+\int|\widehat r_\beta-r_X|d\pi\\
&\le2\int|\widehat r_\beta-r_X|d\pi.
\end{aligned}
\]
Thus the average total variation distance is bounded by $\E_Q|r_X-\widehat r_\beta|$.
Applying $W_2^2\le D_\Theta^2\operatorname{TV}$ proves the claim.
\end{proof}

For unbounded parameter laws, suppose instead that some $r>2$ and $M_r<\infty$ satisfy
\[
\E\norm\theta^r
+\sup_{\eta,\beta}\E_X\int\norm\vartheta^r\widetilde\pi_\beta(d\vartheta\mid f_\eta(X))
\le M_r.
\]
This bounds the prior moment and the average moment of the normalized fitted posterior.

\begin{claim}[Posterior error under a moment bound]
\label{claim:c-nre-moments}
Under this moment condition and the same bound on the fitted classifier,
\[
\E_XW_2^2\{\widetilde\pi_\beta(\cdot\mid S),\pi(\cdot\mid X)\}
\le(2^{r-1}M_r)^{2/r}
\left[\frac{\sqrt{2\{\mathcal L_{\mathrm{NRE}}(\eta,\beta)-\mathcal L_{\mathrm{NRE}}^*\}}}{\tau_n}\right]^{1-2/r}.
\]
\end{claim}

Both bounds supply the concave posterior calibration required by Theorem~\ref{thm:v3-two-way} for the normalized fitted law.
Sampling from that law with a finite proposal pool introduces a further approximation error.

\begin{proof}
Take measurable conditional maximal couplings of the true and normalized fitted posteriors, and let $Y$ be the distance between the paired draws after averaging over $X$.
The preceding proof bounds the mismatch probability by
\[
\Pr(Y>0)\le\frac{\sqrt{2\{\mathcal L_{\mathrm{NRE}}(\eta,\beta)-\mathcal L_{\mathrm{NRE}}^*\}}}{\tau_n}.
\]
The moment condition gives $\E Y^r\le2^{r-1}M_r$.
H\"older's inequality, $\E Y^2\le(\E Y^r)^{2/r}\Pr(Y>0)^{1-2/r}$, bounds the average conditional $W_2^2$ cost and proves the claim.
\end{proof}

\subsubsection{Finite sample error}

We control both estimation of the classification objective and the finite proposal pool used for posterior sampling.
For $0<\tau_n<1/2$, assume $\tau_n\le D_\beta\le1-\tau_n$ uniformly, and write $V_n=\log\{(1-\tau_n)/\tau_n\}$ for the common range length of the positive and negative label losses.
Take $N\ge2$ independent prior--simulator pairs $(\theta_i,X_i)$ and fix the candidate class before simulation.
Let $\sigma$ be a permutation with no fixed points, chosen independently of these pairs.
It may act separately within disjoint batches, as in the implemented nonzero cyclic shifts of freshly simulated batches.
Define the empirical objective
\[
\mathcal L_{\mathrm{NRE},N}(\eta,\beta)
=-\frac1{2N}\sum_{i=1}^N
\left[\log D_\beta(\theta_i,f_\eta(X_i))
+\log\{1-D_\beta(\theta_{\sigma(i)},f_\eta(X_i))\}\right].
\]
Write $\zeta$ for the complete parameter vector of the encoder and classifier, including any fitted normalization parameters, and let $W_n$ be its dimension.
This total parameter count is separate from the encoder and auxiliary-network count $P_n$ in Corollary~\ref{cor:v3-affine-rate}.
Assume the candidate parameter set is contained in $[-L_n,L_n]^{W_n}$ for some $L_n\ge1$.
Define the complete normalized paired loss by
\[
\ell_\zeta(\theta,X,\theta')
=-\frac{\log D_\beta(\theta,f_\eta(X))
+\log\{1-D_\beta(\theta',f_\eta(X))\}}{2V_n},
\]
and assume
\[
\sup_{\theta,X,\theta'}|\ell_\zeta(\theta,X,\theta')-\ell_{\zeta'}(\theta,X,\theta')|
\le L_n\|\zeta-\zeta'\|_\infty
\]
for all candidates $\zeta,\zeta'$.
For this class, set $A_n=1+2L_n^2$.
These conditions count the encoder and classifier together and impose no affine dependence on the summary.
Here $n$ is the observation size, whereas $N$ is the number of simulated training pairs.

\begin{claim}[Statistical error of the classification objective]
\label{claim:c-nre-empirical}
Under these conditions,
\[
\E_N\sup_{\eta,\beta}
\left|\mathcal L_{\mathrm{NRE},N}(\eta,\beta)-\mathcal L_{\mathrm{NRE}}(\eta,\beta)\right|
\le V_n\left[\frac2N+\frac32\sqrt{\frac{\log2+W_n\log(A_nN)}{2N}}\right].
\]
\end{claim}

The statistical term includes the cost of learning the summary and the dependence introduced by reusing simulated records in negative pairs.
For fixed $n$ and a fixed class, it decreases at order $\sqrt{\log N/N}$.

\begin{proof}
Condition on $\sigma$.
Since $\sigma(i)\ne i$, each negative pair has law $Q$, so the empirical objective is unbiased for the population objective.
Replacing one simulated record changes one positive contribution and at most two negative contributions.
Each label loss has range length $V_n$, so the change in the empirical average is at most $3V_n/(2N)$.
The exponential bounded-differences inequality therefore gives, for each fixed candidate and $\lambda\in\R$,
\[
\E_N\exp\left\{\lambda\frac{\mathcal L_{\mathrm{NRE},N}-\mathcal L_{\mathrm{NRE}}}{V_n}\right\}
\le\exp\{9\lambda^2/(32N)\}.
\]
Claim~\ref{claim:b-parametric-cover} bounds the logarithm of the complete-loss covering number by $W_n\log(A_n/\delta)$ for $0<\delta\le1$.
Apply Claim~\ref{claim:b-uniform-deviation} with concentration constant $c=3/4$ and choose $\delta=1/N$.
The two covering errors sum to $2/N$, and the deviation over the representatives is at most $\tfrac32\sqrt{\{\log2+W_n\log(A_nN)\}/(2N)}$.
Multiplying by $V_n$ and averaging over $\sigma$ proves the claim.
\end{proof}

We next bound the error of sampling--importance--resampling with a shared proposal pool.
For fixed networks and summary $s$, draw $K$ independent prior proposals $\vartheta_1,\ldots,\vartheta_K$, independently of training and the observation.
Conditional on this pool, each resampled draw has law
\[
\widehat\pi_{\beta,K}(\cdot\mid s)
=\frac{\sum_{j=1}^K\widehat r_\beta(\vartheta_j,s)\delta_{\vartheta_j}(\cdot)}
{\sum_{j=1}^K\widehat r_\beta(\vartheta_j,s)},
\]
where $\delta_{\vartheta_j}$ is the point mass at $\vartheta_j$.
Assume the prior support has diameter $D_\Theta$.
For $\varepsilon>0$, let $J_\varepsilon$ be the smallest number of nonempty measurable cells partitioning this support such that any two parameter values in the same cell are at most $\varepsilon$ apart.
Fix such a partition independently of the proposal pool.
Thus $J_\varepsilon$ counts the parameter regions needed to resolve the prior support at accuracy $\varepsilon$.
Reducing $\varepsilon$ requires at least as many cells, while $J_\varepsilon=1$ whenever $\varepsilon\ge D_\Theta$.
The count also depends on the size and geometry of the support: for a support contained in $[0,1]^d$, a regular grid gives $J_\varepsilon\le\lceil\sqrt d/\varepsilon\rceil^d$, where $d$ is the parameter dimension.
Write $\E_{\mathrm{pool}}$ for expectation over the independent prior proposals.

\begin{claim}[Posterior error from a finite proposal pool]
\label{claim:c-nre-pool}
Under these conditions, for every $\varepsilon>0$,
\[
\E_{\mathrm{pool}}W_2^2\{\widehat\pi_{\beta,K}(\cdot\mid s),\widetilde\pi_\beta(\cdot\mid s)\}
\le\varepsilon^2+
\frac{D_\Theta^2(1-\tau_n)^2}{2\tau_n^2}\sqrt{\frac{J_\varepsilon}{K}}.
\]
\end{claim}

This bound controls the posterior distribution conditional on the shared pool, averaged over its generation.
The term $\varepsilon^2$ bounds transport within a cell, while $\sqrt{J_\varepsilon/K}$ controls the error in assigning probability mass across cells using $K$ proposals.
A finer partition reduces the first term but can increase the second, so resolving more parameter regions calls for a larger proposal pool.
Increasing the number of resampled draws does not remove the approximation error caused by a fixed pool; the pool itself must grow.

\begin{proof}
Write $w=\widehat r_\beta(\cdot,s)$ and $Z=\int w\,d\pi$ within this proof.
The classifier bounds give $a\le w\le b$, where $a=\tau_n/(1-\tau_n)$ and $b=(1-\tau_n)/\tau_n$.
For partition cells $A_j$, let $p_j=\widetilde\pi_\beta(A_j\mid s)$ and $\widehat p_j=\widehat\pi_{\beta,K}(A_j\mid s)$.
The empirical weight average is at least $a$, and $\E[w(\vartheta)\{\mathbf1_{A_j}(\vartheta)-p_j\}]=0$.
Independence of the proposals and Cauchy--Schwarz give
\[
\E_{\mathrm{pool}}\frac12\sum_j|\widehat p_j-p_j|
\le\frac1{2a}\sqrt{\frac{J_\varepsilon}{K}
\sum_j\E[w(\vartheta)^2\{\mathbf1_{A_j}(\vartheta)-p_j\}^2]}
\le\frac b{2a}\sqrt{\frac{J_\varepsilon}{K}}.
\]
For the last inequality, the sum is at most $bZ(1-\sum_jp_j^2)\le b^2$.
Couple the common probability mass within each cell at squared cost at most $\varepsilon^2$ and couple the remaining mass at squared cost at most $D_\Theta^2$.
Substituting $b/a=(1-\tau_n)^2/\tau_n^2$ proves the claim.
\end{proof}

We now choose the partition resolution to obtain an explicit rate in $K$.
For $D_\Theta>0$, the prior support lies in a box with side lengths at most $D_\Theta$, so a regular grid gives
\[
J_\varepsilon\le\left\lceil\frac{D_\Theta\sqrt d}{\varepsilon}\right\rceil^d.
\]
Choose $\varepsilon=D_\Theta\sqrt d\,K^{-1/(d+4)}$ for $K\ge1$.
This choice balances the powers of $K$ in the within-cell transport and cell-mass errors.
Since $\lceil K^{1/(d+4)}\rceil\le2K^{1/(d+4)}$, Claim~\ref{claim:c-nre-pool} yields
\[
\E_{\mathrm{pool}}W_2^2\{\widehat\pi_{\beta,K}(\cdot\mid s),\widetilde\pi_\beta(\cdot\mid s)\}
\le D_\Theta^2\left[d+\frac{2^{d/2-1}(1-\tau_n)^2}{\tau_n^2}\right]K^{-2/(d+4)}.
\]
If $D_\Theta=0$, the pool error is zero and this inequality holds directly.

Take $\Delta_{\mathcal F}$ and $\xi_N$ in Theorem~\ref{thm:v3-two-way} for the NRE population and empirical objectives above, with $\mathcal L_0=\mathcal L_{\mathrm{NRE}}^*$.
Under the bounded-support conditions and the other hypotheses of that theorem, for the sampler with a pool of size $K$,
\[
\begin{aligned}
R_p\le{}&\frac{2D_\Theta^2}{\tau_n}
\left[2\Delta_{\mathcal F}+2\E_N\xi_N
+4V_n\left\{\frac2N+\frac32\sqrt{\frac{\log2+W_n\log(A_nN)}{2N}}\right\}\right]^{1/2}\\
&+D_\Theta^2\left[2d+\frac{2^{d/2}(1-\tau_n)^2}{\tau_n^2}\right]K^{-2/(d+4)},
\qquad R_f\le2R_p.
\end{aligned}
\]
The first term follows by applying Claim~\ref{claim:c-nre-bounded} to the expected loss excess from Theorem~\ref{thm:v3-two-way} and multiplying by two from the squared triangle inequality.
The second term is twice the explicit pool error above.
This also bounds the posterior error conditional on the shared pool, averaged over pool generation.
Convexity of $W_2^2$ gives the same bound for the randomized sampler after averaging over that pool.
For fixed $n$ and classes with finite $W_n$ and $L_n$, $\Delta_{\mathcal F}=0$ and $\E_N\xi_N\to0$ imply $R_p\to0$ and $R_f\to0$ as $N,K\to\infty$.

For unbounded parameter spaces, Claim~\ref{claim:c-nre-moments} converts the same explicit loss-excess bound into posterior error for exact sampling from the normalized fitted law $\widetilde\pi_\beta$.
A finite proposal pool on an unbounded support additionally requires tail control for its sampling error.
The empirical bounds require uniform classifier bounds and uniform sensitivity of the complete loss to all fitted parameters.
A finite network with a sigmoid output alone does not ensure these conditions on unbounded inputs.

\section{Experimental Details and Additional Results}
\label{app:additional-simulations}

\subsection{Training Settings}
\label{app:experiment-training}

We generate fresh prior-predictive minibatches and train a separate model at each observation size, using training seeds $0,1,2$.
We use minibatches of size $256$ and $20{,}000$ training iterations.
For SB-GAN and Raw B-GAN, each iteration consists of $15$ critic updates and one generator update.
We use Adam with initial learning rate $10^{-3}$ for the non-adversarial methods and constant learning rate $10^{-4}$ for the adversarial methods, whose one-sided gradient penalty has weight $10$.
The raw-data counterparts use the same minibatch size, number of training iterations, and learning rates as their summary-based counterparts.

DGQB uses a scalar input-convex network with three hidden layers of width $64$ and Softplus activations; its coefficient network has three hidden layers of width $64$ and LeakyReLU activations.
The SB-GAN coefficient networks use the latter specification.
Raw DGQB uses a three-layer partially input-convex network with width $64$ in both its convex and observation branches.
The flow-matching velocity networks use hidden widths $(h,2h)$ for the observation or summary block and $(h,2h,4h)$ for the parameter block, with BatchNorm and ELU activations.
The base widths $h$ are $(160,224,288,320,352,384,384)$ for Gaussian and $(224,288,320,352,384)$ for Brock--Hommes, in increasing order of $n$.
Raw B-GAN uses four hidden layers of width $h$ with ELU activations.

\subsection{Gaussian Conjugate Experiment}
\label{app:npe-nre-experiments}

We use the model and evaluation procedure in Section~\ref{sec:new-gaussian}.
The DeepSets encoder uses sum pooling, elementwise hidden width $64$, post-pooling hidden width $16$, and LeakyReLU activations.
Preliminary experiments considered summary dimensions including $q=4,8,16,32$ in stages.
Based on posterior approximation performance, we adopted $q=32$ and held it fixed across observation sizes and summary-based methods.
Training uses standardized $(\mu,\log\sigma^2)$, and posterior draws are transformed back to $(\mu,\sigma^2)$ for evaluation.

For the NPE and NRE results in Figure~\ref{fig:new-gaussian-swd}(b), we use the same DeepSets encoder.
NPE uses five masked autoregressive flow transformations, each with three hidden layers of width $64$ and tanh activations.
NRE uses a classifier with three hidden layers of width $64$ and ELU activations.
Their objectives and posterior sampling procedures are given in Appendices~\ref{app:original-npe} and~\ref{app:original-nre}.

\subsection{Brock--Hommes Experiment}
\label{app:new-bh-contours}

We use the model and initial values in Section~\ref{sec:new-brock-hommes}, with no burn-in.
Observations undergo an elementwise $\operatorname{asinh}$ transformation followed by scalar standardization using prior-predictive simulations; parameters are standardized using their prior means and standard deviations.
The DeepSets branch applies mean pooling to four-entry lag windows and returns $48$ outputs.
The LSTM branch has one layer with $128$ hidden units and a linear head returning $16$ outputs.
Preliminary comparisons of summary dimensions and branch allocations led to this $q=64$ configuration, which is shared by all five summary-based methods across observation sizes.
For flow matching, we solve the ODEs using fourth-order Runge--Kutta on a $64$-point time grid.

Figure~\ref{fig:new-bh-contours} compares all five methods and their raw-data counterparts on the first observed series and training seed $0$.
The left and right panels show $(g_2,b_2)$ and $(g_3,b_3)$, respectively.
Solid and dashed curves bound highest-density regions containing $50\%$ and $90\%$ of each fitted bivariate marginal kernel density estimate (KDE).

\subsubsection*{Interpretation of posterior contours.}
We fit each Gaussian KDE using all $2{,}000$ draws, with Scott's bandwidth multiplied by $1.35$, and calibrate its probability-mass thresholds before restricting the display range.
For visibility, we retain common axes based on the central $99.4\%$ of pooled draws from DGQB, SFM, SB-GAN, and Raw B-GAN.
Some wider distributions, particularly Raw FM and Raw NRE, extend substantially beyond these axes; the visible contour fragments alone do not enclose the stated probability masses.

The plots illustrate changes in posterior location and spread, with clear summary/raw differences for SFM and NPE.
These changes are not uniform across methods, observation sizes, or parameter marginals.
In particular, SB-GAN retains visible displacement from the generating parameter in several panels, DGQB produces very concentrated regions that sometimes exclude it, and NRE does not display monotone contraction as $n$ increases.
A single marginal region excluding the generating parameter does not establish miscalibration, and concentration alone does not establish accurate uncertainty quantification.
The reported masses describe each fitted two-dimensional KDE.
From this experiment alone, we do not intend to establish coverage or the accuracy of the full four-dimensional posterior, and the exact posterior is unavailable in closed form.
We therefore use these visualizations as qualitative diagnostics, with conclusions about relative performance based on the energy-score summaries instead.

\begin{figure}[p]
    \centering
    \includegraphics[width=\textwidth,height=0.78\textheight,keepaspectratio]{\arxivdir 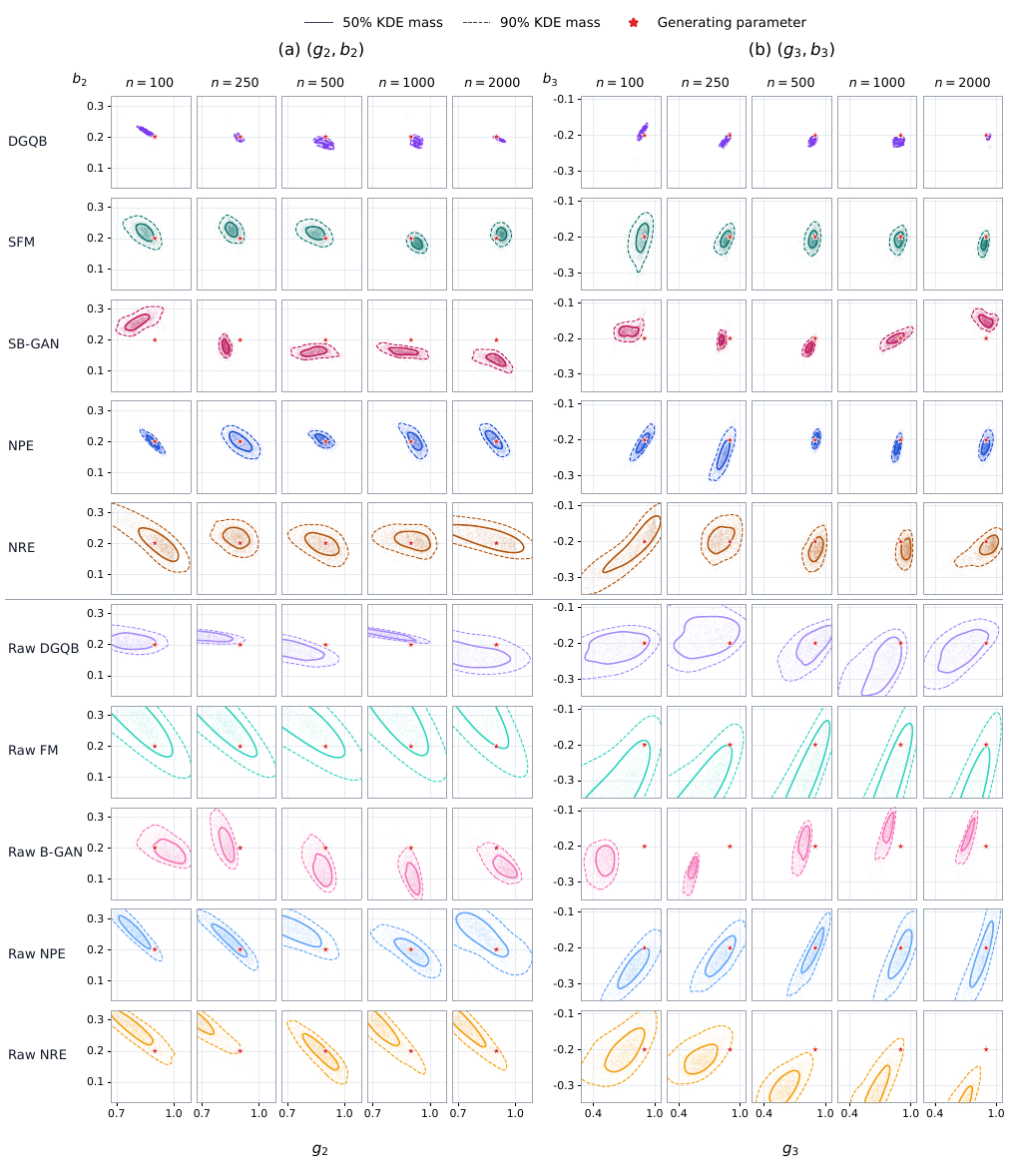}
    \caption{Brock--Hommes posterior draws and marginal KDE contours: (a) $(g_2,b_2)$ and (b) $(g_3,b_3)$.
    Rows show the five summary-based learners followed by their raw-data counterparts; columns within each half increase $n$.
    All panels use the first observed series, training seed $0$, and $2{,}000$ posterior draws.
    Solid and dashed contours enclose $50\%$ and $90\%$ of the full fitted marginal KDE; stars mark the generating parameter.
    Axes are shared within each half, and some contours extend outside the displayed range.}
    \label{fig:new-bh-contours}
\end{figure}
\clearpage

\subsubsection*{Additional NPE and NRE results.}
We jointly train NPE and NRE with the same $q=64$ DeepSets--LSTM encoder described above.
NPE uses five masked autoregressive transformations, each with three hidden layers of width $128$ and tanh activations; NRE uses a classifier with three hidden layers of width $128$ and ELU activations.
The raw versions use the full transformed observation vector as input.
NPE generates posterior draws directly, whereas the reported NRE draws use $64$ parallel Metropolis chains targeting the learned ratio times the known prior, with $1{,}500$ warm-up iterations and $2{,}000$ retained iterations per chain, followed by subsampling to $2{,}000$ draws.
Numerical convergence and effective-sample-size checks screen the NRE sampler and do not certify the accuracy of the learned posterior.

Figure~\ref{fig:new-bh-npe-nre-energy} reports the energy score in \eqref{eq:new-energy-score} for both methods and their raw-data counterparts.
For this additional comparison, all four learners are evaluated on the same first observed series, and we report the mean and sample standard deviation over training seeds $0,1,2$.
Thus, this figure uses one common test series, whereas Figure~\ref{fig:new-bh-energy} first averages over three test series within each seed.
Both NPE and NRE with jointly learned summaries achieve lower mean energy scores than their raw-data counterparts at every observation size considered.
NRE exhibits substantial seed-to-seed variation at $n=500$; all three seeds are retained in the summary.
These results support relative improvements under the reported energy-score criterion, without establishing calibrated uncertainty or consistent posterior recovery for every learner.

\begin{figure}[!ht]
    \centering
    \includegraphics[width=0.48\textwidth]{\arxivdir 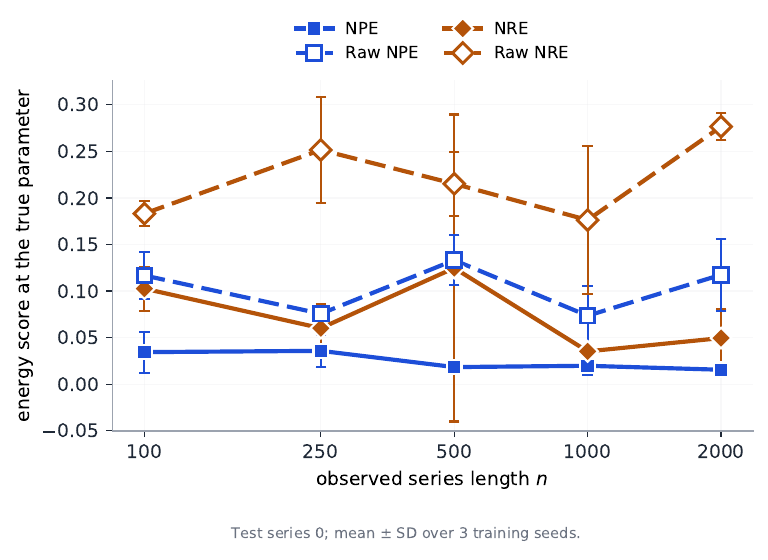}
    \caption{Brock--Hommes energy scores for NPE and NRE on the same first observed series: mean and sample SD across three training seeds, using $2{,}000$ posterior draws per run.
    Solid lines with filled markers denote learned summaries; dashed lines with hollow markers denote raw observations.
    Lower scores are better, and error bars show SD rather than confidence intervals.}
    \label{fig:new-bh-npe-nre-energy}
\end{figure}

\end{document}